\documentclass[11pt]{article}

\makeatletter
\@ifclassloaded{IEEEtran}{
    \def\GenerateShortVersion{}
    
}{}
\@ifclassloaded{article}{
    \usepackage{geometry}
    \geometry{margin=1.1in}
    
    \def\GenerateSingleColumn{}
}{}
\makeatother

\usepackage{ifthen}
\ifthenelse{\isundefined{\GenerateShortVersion}}
{
\def\LongVersion{}
\def\LongVersionEnd{}
\long\def\ShortVersion#1\ShortVersionEnd{}
}
{
\def\ShortVersion{}
\def\ShortVersionEnd{}
\long\def\LongVersion#1\LongVersionEnd{}
}

\ifthenelse{\isundefined{\GenerateSingleColumn}}
{
\def\DoubleColumn{}
\def\DoubleColumnEnd{}
\long\def\SingleColumn#1\SingleColumnEnd{}
}
{
\def\SingleColumn{}
\def\SingleColumnEnd{}
\long\def\DoubleColumn#1\DoubleColumnEnd{}
}

\usepackage{amsthm}
\usepackage{thmtools}
\usepackage{bbm}
\usepackage{complexity}
\usepackage{xspace}
\usepackage{algorithm}
\usepackage{algpseudocode}
\usepackage{amsfonts}
\usepackage{amsmath}
\usepackage[numbers]{natbib}
\usepackage{xcolor}
\usepackage{graphicx}
\usepackage{caption}
\usepackage{subcaption}
\usepackage{authblk}
\usepackage{amssymb}
\usepackage{hyperref}

\newcommand{\code}[1]{\textsf{#1}}
\newcommand{\mln}{\Phi}
\newcommand{\sentence}{\Gamma}
\newcommand{\generalsentence}{{\widehat{\sentence}}}
\newcommand{\recursivesentence}{{\widetilde{\sentence}}}
\newcommand{\fotwoformula}{\psi}
\newcommand{\formula}{\alpha}

\newcommand{\weight}{w}
\newcommand{\vecweight}{\mathbf{\weight}}
\newcommand{\negweight}{\bar{w}}

\newcommand{\world}{\omega}
\newcommand{\wfomc}{WFOMC}

\newcommand{\symwfomc}{\ensuremath{\mathsf{WFOMC}}}

\newcommand{\fotwo}{\ensuremath{\mathbf{FO}^2}\xspace}
\newcommand{\ctwo}{\ensuremath{\mathbf{C}^2}\xspace}
\newcommand{\sctwo}{\ensuremath{\mathbf{SC}^2}}

\newcommand{\indicator}{\mathbbm{1}}

\newcommand{\domain}{\Delta}

\newcommand{\vecn}{\mathbf{n}}

\newcommand{\vecg}{\mathbf{g}}

\newcommand{\vecy}{\mathbf{y}}
\newcommand{\vecx}{\mathbf{x}}

\newcommand{\nat}{\mathbb{N}}
\newcommand{\real}{\mathbb{R}}

\newcommand{\pro}{\mathbb{P}}
\newcommand{\extformula}{\varphi}

\newcommand{\pair}[1]{\{#1\}}

\newcommand{\fomodels}[2]{\mathcal{M}_{#1, #2}}

\newcommand{\ufotwo}{\ensuremath{\mathbf{UFO}^2}\xspace}

\newcommand{\structure}{\mathcal{A}}
\newcommand{\typeweight}[1]{\langle \weight, \negweight\rangle(#1)}

\newcommand{\sumweight}{\mathcal{W}}
\newcommand{\relaxcelltype}[2]{{#1}_{\mid {#2}}}
\newcommand{\proj}[2]{\langle#1\rangle_{#2}}

\usepackage{forloop}
\newcounter{ct}

\newtheorem{definition}{Definition}
\newtheorem{example}{Example}
\newtheorem{remark}{Remark}

\newtheorem{theorem}{Theorem}
\newtheorem{lemma}{Lemma}

\newtheorem{corollary}{Corollary}

\newenvironment{proofs}{%
  \proof}{\endproof}

\algrenewcommand\textproc{\code}

\def\BibTeX{{\rm B\kern-.05em{\sc i\kern-.025em b}\kern-.08em
    T\kern-.1667em\lower.7ex\hbox{E}\kern-.125emX}}
\begin{document}

\title{On Exact Sampling in the Two-Variable Fragment \\of First-Order Logic
}

% \and
% \and
% \and
% \and
% \and
\author[1,2]{Yuanhong Wang}
\author[1,2]{Juhua Pu}
\author[3,4]{Yuyi Wang}
\author[5]{Ond\v{r}ej Ku\v{z}elka}

\affil[1]{State Key Laboratory of Software Development Environment, \protect\\ Beihang University, China}
\affil[2]{Beihang Hangzhou Innovation Institute Yuhang, China}
\affil[3]{CRRC Zhuzhou Institute, China }
\affil[4]{ETH Zurich, Switzerland }
\affil[5]{Czech Technical University in Prague, Czech Republic}

\maketitle

\begin{abstract}
  In this paper, we study the sampling problem for first-order logic proposed recently by Wang et al.---how to efficiently sample a model of a given first-order sentence on a finite domain? We extend their result for the universally-quantified subfragment of two-variable logic $\mathbf{FO}^2$ ($\mathbf{UFO}^2$) to the entire fragment of $\mathbf{FO}^2$. Specifically, we prove the domain-liftability under sampling of $\mathbf{FO}^2$, meaning that there exists a sampling algorithm for $\mathbf{FO}^2$ that runs in time polynomial in the domain size. We then further show that this result continues to hold even in the presence of counting constraints, such as $\forall x\exists_{=k} y: \varphi(x,y)$ and $\exists_{=k} x\forall y: \varphi(x,y)$, for some quantifier-free formula $\varphi(x,y)$. Our proposed method is constructive, and the resulting sampling algorithms have potential applications in various areas, including the uniform generation of combinatorial structures and sampling in statistical-relational models such as Markov logic networks and probabilistic logic programs.
\end{abstract}

% \begin{IEEEkeywords}
% component, formatting, style, styling, insert
% \end{IEEEkeywords}

\section{Introduction}

Let $\sentence$ denote a function-free first-order sentence formed over a vocabulary $\mathcal{P}$, and let $\domain$ be a finite domain.
A model of $\sentence$ interprets each predicate in $\mathcal{P}$ over $\domain$ such that the interpretation satisfies $\sentence$.
We use $\fomodels{\sentence}{\domain}$ to denote the set of all models of $\sentence$ over $\domain$.
The \textit{uniform first-order model sampling problem} on $\sentence$ over $\domain$ is to uniformly generate a model $\mu$ of $\sentence$ according to the probability $\pro[\mu] = 1/|\fomodels{\sentence}{\domain}|$.
The weighted variant of this problem adds nonnegative weights to atomic facts and their negations in the models; the total weight of a model is the product of its facts' weights.
The problem is then to sample a model according to a probability strictly proportional to its weight.

We investigate the \textit{symmetric weighted first-order model sampling problem (WFOMS)} for the two-variable fragment \fotwo{} of first-order logic.
The term ``symmetric'' refers to the property that the weights are determined solely by the relation symbol.
% of the positive and negative fact, and thus can be represented by weighting functions $\weight$ and $\negweight$ that assign weights to each relation symbol for positive ($\weight$) or negative ($\negweight$) facts.
In this paper, we focus on studying the \textit{data complexity} of WFOMS---the complexity of sampling a model when the $\sentence$ and $\weight$ and $\negweight$ are fixed, and the domain is considered as an input.
In particular, we are interested in designing a \textit{domain-lifted} weighted model sampler for \fotwo{}, which runs in time polynomial in the size of the domain.

The WFOMS was first considered by \citet{wangDomainLiftedSamplingUniversal2022} who showed that the data complexity of WFOMS is in polynomial time for formulas of the universally-quantified subfragment \ufotwo{} of \fotwo{}.
The subfragment \ufotwo{}, comprising of sentences of the form $\forall x\forall y: \fotwoformula(x,y)$ with some quantifier-free formula $\fotwoformula(x,y)$, is proved to admit a lifted weighted model sampler, and then identified to be \textit{domain-liftable under sampling}.

Symmetric weighted model sampling problems have a wide range of practical applications.
For example, many problems related to the generation of combinatorial structures can easily be formulated as WFOMS and solved using the techniques developed for this problem.
There are also applications of WFOMS in the realm of \textit{statistical-relational learning (SRL)}~\cite{raedt2016statistical}.
It is known that probabilistic inference in many SRL models is reducible to weighted first-order model counting (WFOMC)~\cite{vandenbroeckLiftedProbabilisticInference2011,vandenbroeck2014Proc.FourteenthInt.Conf.Princ.Knowl.Represent.Reason.}, and the same reduction can also be applied to the corresponding sampling problems.

Among the various applications of WFOMS, the input first-order sentences are usually complex and go beyond the fragment of \ufotwo{}.
For instance, even the very simple problem of uniformly generating graphs with no isolated vertices necessitates the utilization of the existentially-quantified formula $\forall x\exists y: E(x,y)$ to encode the constraint that every vertex must have at least one incident edge.
However, directly extending the approach described in~\cite{wangDomainLiftedSamplingUniversal2022} to \fotwo{} is infeasible.
As the authors showed, their technique would at some point need to solve \#\P-hard problems: ``\textit{\dots applying our sampling algorithm on an \fotwo{} sentence with existential quantifiers is intractable (not domain-lifted) unless \FP=\#\P,\dots}''.
We stress here that they did not show the intractability of WFOMS for the \fotwo{} fragment, but rather the infeasibility of their specific method, indicating that a distinct sampling approach is required.
Moreover, the standard Skolemization techniques used in automated reasoning~\cite{robinson2001handbook} and WFOMC~\cite{vandenbroeck2014Proc.FourteenthInt.Conf.Princ.Knowl.Represent.Reason.} to eliminate existential quantifiers beforehand are not applicable to WFOMS, as they introduce either functions or negative weights, which make the resulting sampling problem ill-defined.
This further complicates the extension of the WFOMS approach to more complex formulas beyond \ufotwo{}.

\subsection{Our Contribution}

In this paper, we present a novel sampling algorithm for the full \fotwo{}.
The algorithm employs a completely different approach than Skolemization,  based on the \textit{domain recursion} scheme.
The basic idea is to consider one object from the domain at a time, and then sample the value of all related atomic facts, resulting in a new WFOMS over a smaller domain with the object removed. 
The new WFOMS has an identical form to the original one but possibly contains fewer existentially-quantified formulas.
The algorithm then runs recursively on the reduced sampling problems until the domain becomes singleton or all existentially-quantified formulas are eliminated.
We prove that the data complexity of our algorithm is in PTIME, meaning that the entire fragment of \fotwo{} is domain-liftable under sampling.

We also show how to further extend the result to the cases, where we include \textit{counting constraints}.
Specifically, our generalized algorithm can be applied to the \fotwo{} sentences with additional counting constraints of the form $\forall x\exists_{=k} y: \extformula(x,y)$ and $\exists_{=k} x\forall y: \extformula(x,y)$, where $\extformula$ is a quantifier-free formula and $k$ is a natural number.
This extension, originally proposed by ~\citet{kuusistoWeightedModelCounting2018} and~\citet{kuzelkaWeightedFirstorderModel2021} for first-order counting problems, is mainly motivated by the connection of WFOMS to the uniform generation of combinatorial structures.
For example, our algorithm can be applied to efficiently solve the uniform sampling problem of $k$-regular graphs, a problem that has been widely studied in the combinatorics community~\cite{cooperSamplingRegularGraphs2007,gaoUniformGenerationRandom2015}.
This problem can be formulated as a WFOMS on the following sentence:
\DoubleColumn
\begin{align*}
  &\forall x\forall y: (E(x,y)\Rightarrow E(y,x))\land \\
  &\forall x: \neg E(x,x)\land \\
  &\forall x\exists_{=k}y: E(x,y),
\end{align*}
\DoubleColumnEnd
\SingleColumn
\begin{equation*}
  \forall x\forall y: (E(x,y)\Rightarrow E(y,x))\land \forall x: \neg E(x,x)\land \forall x\exists_{=k}y: E(x,y),
\end{equation*}
\SingleColumnEnd
where $\forall x\exists_{=k} y: E(x,y)$ expresses that every vertex $x$ has exactly $k$ connected edges. 

\subsection{Related Work}

% 1. UFO2 paper as well as the lifted inference.
% 2. domain recursion
% 3. sampling on propositional level
% 4. construction problem of FO2

The symmetric weighted first-order model sampling problem was first proposed and studied in~\cite{wangDomainLiftedSamplingUniversal2022}.
The approach, as well as the formal \textit{liftability} notions considered in that study, were derived from the literature on \textit{lifted inference}~\cite{poole2003first,richardsonMarkovLogicNetworks2006,vandenbroeckLiftedProbabilisticInference2011}.
In lifted inference, the goal is to perform probabilistic inference in SRL models in a way that takes advantage of the \textit{symmetries} in the high-level structure of the models.
The symmetry also exists in WFOMS and is a vital property leveraged by this paper to prove the liftability under sampling of \fotwo{}.
We note here that the importance of symmetry for lifted inference (and its reduced WFOMC) has also been extensively discussed by~\citet{beameSymmetricWeightedFirstOrder2015}.

The domain recursion approach adopted in this paper is similar to the \textit{domain recursion rule} used in weighted first-order model counting~\cite{broeckCompletenessFirstOrderKnowledge2011,kazemi2016new,kazemiDomainRecursionLifted2017,toth2022ArXivPrepr.ArXiv221101164}.
The domain recursion rule for \wfomc{} is a technique that utilizes a gradual grounding process on the input first-order sentence, where only one element of the domain is grounded at a time. 
As each element is grounded, the partially grounded sentence is simplified until the element is entirely removed, resulting in a new \wfomc{} problem with a smaller domain. 
With the domain recursion rule, one can apply the principle of induction on the domain size, and compute \wfomc{} by dynamic programming.
A closely related work to this paper is the approach presented by \citet{kazemi2016new}, where they used the domain recursion rule to compute WFOMC without Skolemization~\cite{vandenbroeck2014Proc.FourteenthInt.Conf.Princ.Knowl.Represent.Reason.}, which introduces negative weights.
However, it is important to note that their approach can be only applied to some specific first-order formulas, whereas the domain recursion scheme presented in this paper, mainly designed for eliminating the existentially-quantified formulas, supports the entire \fotwo{} fragment.

It is also worth mentioning that sampling from \textit{propositional} logic formula (Boolean formula) is a relatively well-studied area~\cite{gomesModelCounting2009,chakrabortyScalableNearlyUniform2013,chakrabortyParallelScalableUniform2015}.
However, many real-world problems can be represented more naturally and concisely in first-order logic, and suffer from a significant increase in formula size when grounded out to propositional logic. 
For example, a formula of the form $\forall x\exists y: \extformula$ is encoded as a Boolean formula of the form $\bigwedge_{i=1}^n\bigvee_{j=1}^n l_{i,j}$, whose length is quadratic in the domain size $n$. 
Since even finding a solution to a such large ground formula is challenging, most sampling approaches for propositional logic instead focus on designing approximate samplers. 
We also note that these approaches are not polynomial-time in the length of the input formula, and rely on access to an efficient SAT solver.
An alternative strand of research~\cite{guo2019uniform,he2021perfect,feng2021sampling} on combinatorial sampling, focuses on the development of near-uniform and efficient sampling algorithms. 
However, these approaches can only be employed for specific Boolean formulas that satisfy a particular technical requirement known as the Lovász Local Lemma. 
The WFOMS problems studied in this paper do not typically meet the requisite criteria for the application of these techniques.

\section{Preliminaries}

In this section, we briefly review the main necessary technical concepts that we will use in the paper

\subsection{Symmetric weighted first-order model sampling}

We consider the function-free fragment of first-order logic.
An \emph{atom} of arity $k$ takes the form $P(x_1,\dots,x_k)$ where $P/k$ is from a vocabulary of predicates (also called relations), and $x_1,\dots,x_k$ are logical variables from a vocabulary of variables.
A \emph{literal} is an atom or its negation. 
A \emph{formula} is formed by connecting one or more literals together using conjunction or disjunction. 
A formula may optionally be surrounded by one or more quantifiers of the form $\forall x$ or $\exists x$, where $x$ is a logical variable. 
A logical variable in a formula is said to be \emph{free} if it is not bound by any quantifier.
A formula with no free variables is called a \emph{sentence}. 
The vocabulary of a formula $\formula$ is taken to be $\mathcal{P}_\formula$.
% A formula is \emph{ground} if it contains no logical variables. 
% The \emph{groundings} of a quantifier-free formula is the set of formulas obtained by instantiating the free variables with any possible combination of constants from $\domain$.
% A \emph{possible world} $\world$ is represented as the set of ground atoms that are true in $\world$.
% The set of all true (resp. false) atoms in a possible world $\world$ is denoted by $\world_T$ (resp. $\world_F$).
% A possible world $\world$ is called a \textit{model} of a sentence $\sentence$ if it satisfies $\sentence$ (i.e., $\world \models \sentence$), according to the usual semantics of first-order logic. 

Given a vocabulary $\mathcal{P}$, a $\mathcal{P}$-structure $\structure$ interprets each predicate in $\mathcal{P}$ over a given domain.
We often interchangeably view a structure as a set of ground literals and their conjunction.
Given a $\mathcal{P}$-structure $\structure$ and $\mathcal{P}'\subseteq\mathcal{P}$, we write $\proj{\structure}{\mathcal{P}'}$ for the $\mathcal{P}'$-reduct of $\structure$.
We follow the standard semantics of first-order logic for determining whether a structure is a model of a formula.
We denote the set of all models of a sentence $\sentence$ over the domain $\domain$ by $\fomodels{\sentence}{\domain}$. 
% We call a subset $\mathcal{P}$ of the vocabulary $\mathcal{P}_\sentence$ of a sentence $\sentence$ an \textit{independent support} of $\sentence$, if for any domain $\domain$ and any model $\mu$ of $\sentence$ over $\domain$, $\mathcal{P}$ in $\mu$ fully determines $\mathcal{P}_\sentence \setminus\mathcal{P}$.
The two-variable syntactic fragment of first-order logic (\fotwo{}) is obtained by restricting the variable vocabulary to $\pair{x,y}$. 
% \ctwo{} is an extension to \fotwo{} by introducing the counting quantifiers ``exists exactly $k$'' $\exists_{=k}$.
% We will provide the details of $\ctwo{}$ later in the paper.

The \textit{first-order model counting problem}~\cite{vandenbroeckLiftedProbabilisticInference2011} asks, when given a domain $\domain$ and a sentence $\sentence$, how many models $\sentence$ has over $\domain$.
The \textit{weighted first-order model counting problem (WFOMC)} adds a pair of weighting functions $(\weight, \negweight)$ to the input, that both map the set of all predicates in $\sentence$ to a set of weights: $\mathcal{P}_\sentence \to \real$.
Given a set $L$ of literals, the weight of $L$ is defined as
\begin{equation*}
  \typeweight{L} := \prod_{l\in L_T}\weight(\mathsf{pred}(l)) \cdot \prod_{l\in L_F}\negweight(\mathsf{pred}(l))
\end{equation*}
where $L_T$ (resp. $L_F$) denotes the set of true ground (resp. false) literals in $L$, and $\mathsf{pred}(l)$ maps a literal $l$ to its corresponding predicate name.
The value of $\symwfomc(\sentence, \domain, \weight, \negweight)$ is then the sum of the weight $\typeweight{\mu}$ over all models of $\sentence$ over $\domain$.

Recently, the model counting problem was extended to the sampling regime by~\cite{wangDomainLiftedSamplingUniversal2022}, and the \textit{symmetric weighted first-order model sampling problem (WFOMS)} defined therein is the main focus of this paper.

\begin{definition}[\textbf{Symmetric weighted first-order model sampling}]
    \label{def:swfos}
  Let $(\weight, \negweight)$ be a pair of weighting function: $\mathcal{P}_\sentence\to\real_{\ge 0}$~\footnote{The non-negative weights ensures that the sampling probability of a model is well-defined.}.
  The symmetric weighted first-order model sampling problem on $\sentence$ over a domain $\domain$ under $(\weight, \negweight)$ is to generate a model $G(\sentence, \domain, \weight, \negweight)$ of $\sentence$ over $\domain$ such that
  \begin{equation}
      \label{eq:WMS}
      \begin{aligned}
      \pro[G(\sentence, \domain, \weight, \negweight)=\mu] = \frac{\typeweight{\mu}}{\symwfomc(\sentence, \domain, \weight, \negweight)}
      \end{aligned}
  \end{equation}
  for every $\mu\in\fomodels{\sentence}{\domain}$.
\end{definition}

Following the terminology in~\cite{wangDomainLiftedSamplingUniversal2022},  we call a probabilistic algorithm that realizes a solution to the WFOMS a \textit{weighted model sampler} (WMS).
A WMS is \textit{domain-lifted} (or simply lifted) if the model generation algorithm runs in time polynomial in the size of the domain $\domain$.
A sentence, or class of sentences, is \textit{domain-liftable} (or simply \textit{liftable}) \textit{under sampling} if it admits a domain-lifted WMS.
% Several sampling problems in combinatorics can be directly solved given access to a UMS or WMS, as long as the characterizing properties of the structure in the question can be encoded as a first-order logic sentence.

\begin{example}
  \label{ex:non-isolated_graph}
  The WMS of the sentence 
  \begin{equation*}
    \left(\forall x\forall y: (E(x,y)\Rightarrow E(y,x))\land \neg E(x,x)\right)\land \left(\forall x\exists y E(x,y)\right)
  \end{equation*}
  over a domain of size $n$ under the weighting $\weight(E) = \negweight(E) = 1$ uniformly samples undirected graphs with no isolated vertices. 
\end{example}

For technical purposes, when the domain is fixed, we allow the input sentence of the \wfomc{} (and WFOMS) to contain some ground literals, e.g., $\sentence = (\forall x\forall y: fr(x,y) \land sm(x) \Rightarrow sm(y))\land sm(e_1)\land \neg sm(e_3)$ over a fixed domain of $\{e_1, e_2, e_3\}$.
The \wfomc{} problem on such sentences is also known as \textit{conditional} WFOMC~\cite{vandenbroeckConditioningFirstorderKnowledge2012,vandenbroeckComplexityApproximationBinary2013}.
We define the probability of a sentence $\Phi$ conditional on another sentence $\sentence$ over a domain $\domain$ under $(\weight, \negweight)$ as
\begin{equation*}
  \pro[\Phi\mid\sentence; \domain, \weight, \negweight] := \frac{\symwfomc(\Phi\land\sentence, \domain, \weight, \negweight)}{\symwfomc(\sentence, \domain, \weight, \negweight)}.
\end{equation*}
With a slight abuse of notation, we also write the probability of a set $L$ of ground literals conditional on a sentence $\sentence$ over a domain $\domain$ under $(\weight, \negweight)$ in the same form:
\begin{equation*}
  \pro[L \mid \sentence; \domain, \weight, \negweight] := \pro\left[\bigwedge_{l\in L}l \mid \sentence; \domain, \weight, \negweight\right].
\end{equation*}
Then, the required sampling probability of $G(\sentence, \domain, \weight, \negweight)$ in the WFOMS can be written as $\pro[G(\sentence, \domain, \weight, \negweight) = \mu] = \pro[\mu\mid \sentence; \domain, \weight, \negweight]$.
When the context is clear, we omit $\domain$ and $(\weight, \negweight)$ in the conditional probability.

We call a set $L$ of ground literals \textit{valid} in a WFOMS $(\sentence, \domain, \weight, \negweight)$, if there exists a model $\mu\in\fomodels{\sentence}{\domain}$ that includes $L$.
A \textit{skeleton} of a WFOMS $(\sentence, \domain, \weight, \negweight)$ is a subset $\mathcal{P}$ of the vocabulary $\mathcal{P}_\sentence$, such that the interpretation for $\mathcal{P}$ fully determines $\mathcal{P}_\sentence \setminus\mathcal{P}$ in the models of $\sentence$, and for any predicate $P\in\mathcal{P}_\sentence\setminus\mathcal{P}$, $\weight(P) = \negweight(P)= 1$.
Using the notion of skeleton, a WFOMS $(\sentence,\domain,\weight, \negweight)$ can be reduced to randomly generating a valid $\mathcal{P}$-structure $G(\sentence, \domain, \weight, \negweight)$ such that
\begin{equation*}
  \pro[G(\sentence, \domain,\weight, \negweight) = \proj{\mu}{\mathcal{P}}] = \pro[\proj{\mu}{\mathcal{P}}\mid \sentence; \domain, \weight, \negweight]
\end{equation*}
for every $\mu\in\fomodels{\sentence}{\domain}$, where $\mathcal{P}$ is a skeleton of the problem.

In this paper, we often convert complicated WFOMS problems into simpler ones, which are commonly referred to as \textit{reductions}. 
The essential property of such reductions is \textit{soundness}.
\begin{definition}[\textbf{Soundness}]
  \label{de:soundness}
  A reduction of a WFOMS of $(\sentence, \domain, \weight, \negweight)$ to $(\sentence', \domain', \weight', \negweight')$ is sound iff there exists a polynomial-time deterministic function $f$ mapping from $\fomodels{\sentence'}{\domain'}$ to $\fomodels{\sentence}{\domain}$, and for every model $\mu\in\fomodels{\sentence}{\domain}$, 
  \begin{equation}
    \pro[\mu\mid \sentence; \domain, \weight, \negweight] = \sum_{\substack{\mu'\in\fomodels{\sentence'}{\domain'}:\\ f(\mu') = \mu}}\pro[\mu'\mid\sentence'; \domain', \weight', \negweight'].
  \end{equation}
\end{definition}

A general mapping function $f$ used most in this paper is the projection $f(\mu') = \proj{\mu'}{\mathcal{P}_\sentence}$, where $\mathcal{P}_\sentence$ is a skeleton of $(\sentence', \domain', \weight', \negweight')$.
In this case, the mapping function is bijective and preserves the weight of the mapped models.
Through a sound reduction, we can easily transform a WMS $G'$ of $(\sentence', \weight', \negweight', \domain')$ to a WMS $G$ of $(\sentence, \weight, \negweight, \domain)$ by $G(\sentence, \domain, \weight, \negweight) = f(G'(\sentence', \domain', \weight', \negweight'))$.
Note that the soundness is transitive, i.e., if the reductions from a WFOMS $\mathfrak{S}_1$ to $\mathfrak{S}_2$ and from $\mathfrak{S}_2$ to $\mathfrak{S}_3$ are both sound, the reduction from $\mathfrak{S}_1$ to $\mathfrak{S}_3$ is also sound.

\subsection{Types and Tables}

We define a \textit{1-literal} as an atomic predicate or its negation using only the variable $x$, and a \textit{2-literal} as an atomic predicate or its negation using both variables $x$ and $y$.
An atom like $R(x,x)$ or its negation is considered a 1-literal, even though $R$ is a binary relation.
A 2-literal is always of the form $R(x,y)$ and $R(y,x)$, or their respective negations.

Let $\mathcal{P}$ be a finite vocabulary.
A \textit{1-type} over $\mathcal{P}$ is a maximally consistent set of 1-literals formed by $\mathcal{P}$.
Denote the set of all 1-types over $\mathcal{P}$ as $U_\mathcal{P}$.
The size of $U_\mathcal{P}$ is finite and only depends on the size of $\mathcal{P}$.
We often view a 1-type $\tau$ as a conjunction of its elements, whence $\tau(x)$ is simply a formula in the single variable $x$.

Let $\structure$ be a structure over $\mathcal{P}$. 
A domain element $e \in \code{dom}(\structure)$ \textit{realizes} the 1-type $\tau$ if $\structure\models\tau(e)$.
Note that every element of $\structure$ realizes exactly one 1-type over $\mathcal{P}$, which we call the \textit{1-type of the element}.
The \textit{cardinality} of a 1-type is the number of elements realizing it.

A \textit{2-table} over $\mathcal{P}$ is a maximally consistent set of 2-literals formed by $\mathcal{P}$.
We often identify a 2-table $\pi$ with a conjunction of its elements and write it as a formula $\pi(x,y)$.
Denote $T_\mathcal{P}$ the set of all 2-tables over $\mathcal{P}$, whose size also only depends on the size of $\mathcal{P}$.
Given a $\mathcal{P}$-structure $\structure$ over a domain $\domain$, the \textit{2-table of an element tuple} $(a,b) \in \domain^2$ is the unique 2-table $\pi$ that $(a, b)$ satisfies in $\structure$: $\structure\models \pi(a,b)$.
It is worth noting that the 1-types together with the 2-tables fully characterize a structure.

\begin{example}
  Consider the vocabulary $\mathcal{P} = \{F/2, G/1\}$ and the structure 
  \DoubleColumn
  $\{F(a,a), G(a), \neg F(b,b), G(b), F(a,b), \linebreak \neg F(b,a)\}$ 
  \DoubleColumnEnd
  \SingleColumn
  \begin{equation*}
      \{F(a,a), G(a), \neg F(b,b), G(b), F(a,b), \linebreak \neg F(b,a)\} 
  \end{equation*}
  \SingleColumnEnd
  over the domain $\{a, b\}$.
  The 1-type of the elements $a$ and $b$ are $F(x,x)\land G(x)$ and $\neg F(x,x)\land G(x)$ respectively.
  The cardinalities of these two 1-types are both $1$, while that of the other 1-types $F(x,x)\land\neg G(x)$ and $\neg F(x,x)\land \neg G(x)$ are both $0$.
  The 2-table of the element tuples $(a,b)$ and $(b,a)$ are $F(x,y) \land \neg F(y,x)$ and $\neg F(x,y) \land F(y,x)$ respectively.
\end{example}

% Let $\gamma$ be either a 1-type or a 2-table over $\mathcal{P}$. 
% Let $\gamma_T$ and $\gamma_F$ be the sets of positive and negative literals in $\gamma$. 
% With a slight abuse of notations, we  the weight of $\gamma$, denoted by $\typeweight{\gamma}$, is the product $\prod_{l\in\gamma_T}\weight(\code{pred}(l))\cdot\prod_{l\in\gamma_F}\negweight(\code{pred}(l))$.

\subsection{Universally Quantified \fotwo{} is Liftable under Sampling}

As an elementary attempt to the symmetric weighted first-order model sampling problem, \citet{wangDomainLiftedSamplingUniversal2022} provided a positive result of the data complexity for the \textit{universally quantified fragment of \fotwo{}} (\ufotwo) of the form $\forall x\forall y: \fotwoformula(x,y)$, where $\fotwoformula(x,y)$ is a quantifier-free formula\footnote{They went a bit beyond this fragment, e.g., \ufotwo{} with cardinality constraints, which we also handle later in this paper.}.

The proof of this result established a general framework for designing a WMS.
Therefore, We summarize the main ideas of their argument here and refer the reader to their paper for the complete proof and technical details.
We note that the approach presented here is slightly different from the original one in~\cite{wangDomainLiftedSamplingUniversal2022}.
The main divergence is that, instead of using the notion of count distribution~\cite{kuzelkaWeightedFirstorderModel2021}, we perform the sampling of 1-types by a random partition on the domain, which keeps in line with our sampling algorithm for \fotwo{}.
\begin{theorem}[Proposition~1 in \cite{wangDomainLiftedSamplingUniversal2022}]
  \label{th:ufo_liftable}
  The fragment $\ufotwo{}$ is domain-liftable under sampling.
\end{theorem}

\begin{proofs}
Suppose that we wish to randomly sample models from some input \ufotwo{} sentence $\sentence=\forall x\forall y:\fotwoformula(x,y)$ over a domain $\domain = \{e_1, e_2, \dots, e_n\}$ under weights $(\weight, \negweight)$.
Given a $\mathcal{P}_\sentence$-structure $\structure$ over $\domain$, we denote $\tau_i$ the 1-type of the $i$th element and $\pi_{i,j}$ the 2-table of the tuple of the $i$th and $j$th elements.
The structure $\structure$ is fully characterized by the ground 1-types $\tau_i(e_i)$ and 2-tables $\pi_{i,j}(e_i, e_j)$.
We can write the sampling probability of $\structure$ as
\DoubleColumn
\begin{align*}
  &\pro[\structure\mid\sentence] = \pro\left[\bigwedge_{i\in[n]}\tau_i(e_i)\land\bigwedge_{i,j\in[n]: i<j}\pi_{i,j}(e_i, e_j)\mid\sentence\right]=\\
  &\underbrace{\pro\left[\bigwedge_{i\in[n]}\tau_i(e_i)\mid \sentence\right]}_{\mathfrak{P}_1}\cdot \underbrace{\pro\left[\bigwedge_{\substack{i,j\in[n]:\\ i<j}}\pi_{i,j}(e_i, e_j)\mid \sentence\land \bigwedge_{i\in[n]}\tau_i(e_i)\right]}_{\mathfrak{P}_2},
\end{align*}
\DoubleColumnEnd
\SingleColumn
\begin{align*}
  \pro[\structure\mid\sentence] &= \pro\left[\bigwedge_{i\in[n]}\tau_i(e_i)\land\bigwedge_{i,j\in[n]: i<j}\pi_{i,j}(e_i, e_j)\mid\sentence\right]\\
  &=\underbrace{\pro\left[\bigwedge_{i\in[n]}\tau_i(e_i)\mid \sentence\right]}_{\mathfrak{P}_1}\cdot \underbrace{\pro\left[\bigwedge_{\substack{i,j\in[n]:\\ i<j}}\pi_{i,j}(e_i, e_j)\mid \sentence\land \bigwedge_{i\in[n]}\tau_i(e_i)\right]}_{\mathfrak{P}_2},
\end{align*}
\SingleColumnEnd
where $[n]$ denotes the set of $\{1, 2,\dots, n\}$.
This decomposition naturally gives rise to a two-phase sampling algorithm: 
\begin{enumerate}
  \item sample 1-types $\tau_i$ according to the probability $\mathfrak{P}_1$, and
  \item randomly assign 2-table $\pi_{i,j}$ to each element tuple according to $\mathfrak{P}_2$.
\end{enumerate}
The sampling of 1-types can be achieved through a random partition of the domain $\domain$, resulting in $|U_{\mathcal{P}_\sentence}|$ disjoint subsets of $\domain$; each subset contains the elements assigned to its corresponding 1-type.
The symmetry property of the weighting function ensures that the satisfaction and weight of the models are not affected by permutations of the domain elements. Therefore, any partitions of the domain with the same partition size have the same probability to be sampled.
This further decomposes the sampling problem of 1-types into two stages: the stochastic generation of partition size and the random partitioning of the domain according to the sampled size.
Randomly partitioning the domain according to the sampled size is straightforward, and we will demonstrate that sampling a partition size can be done in time polynomial in the domain size.

Recall that the number $|U_{\mathcal{P}_\sentence}|$ of all 1-types only depends on the input sentence, and thus enumerating all possible partition sizes is computationally tractable.
For any partition size $(n_1, n_2, \dots, n_{|U_{\mathcal{P}_\sentence}|})$, there are totally $\binom{n}{n_1,n_2,\dots,n_{|U_{\mathcal{P}_\sentence}|}}$ partitions of the domain with the same sampling probability.
It will turn out that the sampling probability, which is of the form $\mathfrak{P}_1$, can be computed in time polynomial in the domain size.
The reason for this is: we expand $\mathfrak{P}_1$ into $\symwfomc(\sentence\land\bigwedge_{i=1}^n \eta_i(e_i),\domain,\weight, \negweight) / \symwfomc(\sentence, \domain, \weight, \negweight)$; the numerator \wfomc{} can be viewed as a \wfomc{} of $\sentence$ conditional on the unary facts in all $\eta_i(e_i)$, whose complexity is polynomial in the domain size by~\cite{vandenbroeckConditioningFirstorderKnowledge2012}; and the denominator \wfomc{} can be also efficiently computed due to the liftability (in terms of \wfomc{}) of $\sentence$ by~\cite{vandenbroeckLiftedProbabilisticInference2011}.
Finally, the sampling probability of the partition size $(n_1,n_2,\dots,n_{|U_{\mathcal{P}_\sentence}|})$ is given by $\mathfrak{P}_1\cdot \binom{n}{n_1,n_2,\dots,n_{|U_{\mathcal{P}_\sentence|}}}$.

For sampling $\pi_{i,j}$ according to $\mathfrak{P}_2$,
we first ground out $\sentence$ over the domain $\domain$: 
\begin{equation*}
    \bigwedge_{i,j\in[n]: i<j}\fotwoformula(e_i, e_j)\land \fotwoformula(e_j, e_i).
\end{equation*}
Let $\fotwoformula'_{i,j}(x,y)$ be the simplified formula of $\fotwoformula(x,y)\land \fotwoformula(y,x)$ obtained by replacing the unary ground literals with their truth value given by the 1-types $\tau_i$ and $\tau_j$.
Then the probability $\mathfrak{P}_2$ can be written as
\begin{equation*}
    \pro[\bigwedge_{i,j\in[n]:i<j}\pi_{i,j}(e_i, e_j)\mid \bigwedge_{i,j\in[n]: i<j}\fotwoformula_{i,j}'(e_i, e_j)].
\end{equation*}
Note that in this probability, each ground 2-tables $\pi_{i,j}(e_i,e_j)$ are independent in the sense that they do not share any ground literals.
The independence also holds for the ground formulas $\fotwoformula_{i,j}'(e_i, e_j)$.
It follows that the probability $\mathfrak{P}_2$ can be factorized into
\begin{equation*}
  \prod_{i,j\in[n]: i<j} \pro[\pi_{i,j}(e_i,e_j)\mid\fotwoformula_{i,j}'(e_i, e_j)].
\end{equation*}
Hence, the sampling of each $\pi_{i,j}$ can be solved separately by randomly choosing a model of its respective ground formula $\fotwoformula'(e_i, e_j)$ according to the probability $\pro[\pi_{i,j}(e_i,e_j)\mid\fotwoformula_{i,j}'(e_i, e_j)]$.
% \lnote{Yes, I defined it just above Def.~\ref{def:swfos}}
The overall computational complexity is clearly polynomial in the domain size.

The procedure for both sampling $\tau_i$ and $\pi_{i,j}$ is polynomial in the domain size, which completes the liftability under sampling of \ufotwo{}.
\end{proofs}

Extending the approach above to the case of \fotwo{} would requires a novel and more sophisticated strategy, especially for the sampling of 2-tables, as decoupling the grounding of $\forall x\exists y: \extformula(x,y)$ to the form of $\bigwedge_{i,j\in[n]:i<j} \fotwoformula_{i,j}(e_i,e_j)$ is impossible even conditioning on the sampled 1-types.
% One would wonder whether the rejection sampling could come into play.
% \lnote{Remove this claim unless we can give a simple counterexample in a few words..}
% The answer is negative following from a simple counterexample of sampling from the \fotwo{} sentence $\forall x\exists y: E(x,y)\land red(x)$, which can be expressed to that for any vertex in a colored graph, it is either non-isolated or colored by red.
% The rejection probability, proportional to the fraction of random colored graphs that contains at least one vertex being neither non-isolated nor colored by red, 

\subsection{Notations}
We will use $[n]$ to denote the set of $\{1, 2, \dots, n\}$.
The notation $\{x_i\}_{i\in[n]}$ represents the set of terms $\{x_1, x_2, \dots, x_n\}$, and $(x_i)_{i\in[n]}$ the vector of $(x_1, x_2,\dots, x_n)$.
We also use the bold symbol $\vecx$ to denote a vector $(x_i)_{i\in[n]}$, and denote by $\vecx^\vecy$ the product over element-wise power of two vectors $\vecx^\vecy = \prod_{i\in[n]}x_i^{y_i}$.
The notation $\oplus$ is used to denote the concatenation of two vectors.
Using the vector notation, we often write the multinomial coefficient $\binom{N}{x_1, x_2, \dots, x_n}$ as $\binom{N}{\vecx}$.

\section{Sampling Algorithm for \fotwo{}}

We now show the domain-liftability under sampling of the \fotwo{} fragment by providing a lifted WMS for it.
It is common for logical algorithms to operate on normal form representations instead of arbitrary sentences.
The normal form of \fotwo{} used in our sampling algorithm is the \textit{Scott normal form (SNF)}~\cite{scott1962decision}; an \fotwo{} sentence is in SNF, if it is written as:
\begin{equation}
  \label{eq:scott_form}
  \sentence = \forall x\forall y: \fotwoformula(x,y)\land\bigwedge_{k \in[m]} \forall x\exists y: \extformula_k (x,y),
\end{equation}
where $\fotwoformula$ and $\extformula_k$ are quantifier-free formulas.
It is well-known that one can convert any \fotwo{} sentence $\sentence$ in polynomial time into a formula $\sentence_S$ in SNF such that $\sentence$ and $\sentence_S$ are equisatisfiable~\cite{gradel1997Bull.Symb.Log.}.
The principal idea is to substitute, starting from the atomic level and working upwards, any subformula $\varphi(x) = Qy: \phi(x,y)$, where $Q\in\{\forall,\exists\}$ and $\phi$ is quantifier-free, with an atomic formula $A_{\varphi}$, where $A_\varphi$ is a fresh predicate symbol.
This novel atom $A_\varphi(x)$ is then separately ``axiomatized'' to be equivalent to $\varphi(x)$.
The weight of $A_\varphi$ is set to be $\weight(A_\varphi) = \negweight(A_\varphi) = 1$.
It follows from reasoning similar to one by~\citet{kuusistoWeightedModelCounting2018} that such reduction is not only equisatisfiable but also sound (according to Definition~\ref{de:soundness}).
\begin{lemma}
  \label{le:snf_sound}
  For any WFOMS of $\mathfrak{S} = (\sentence, \domain, \weight, \negweight)$ where $\sentence$ is an \fotwo{} sentence, there exists a WFOMS $\mathfrak{S}' = (\sentence', \domain, \weight', \negweight')$, where $\sentence'$ is in SNF and independent of $\domain$, such that the reduction from $\mathfrak{S}$ to $\mathfrak{S}'$ is sound.
\end{lemma}
The proof is clear, as every novel predicate (e.g., $P_\varphi$) introduced in the SNF transformation is axiomatized to be equivalent to the subformula ($\varphi(x)$), whose quantifiers are to be eliminated, and thus the interpretation of the predicate is fully determined by the subformula in every model of the resulting SNF sentence (see the details in Appendix~\ref{sub:snf}).

\subsection{An Intuitive Example}

We start with an intuitive example of how to generate an undirected graph of size $n$ without any isolated vertex uniformly at random, to illustrate the basic idea of our sampling algorithm.
This graph structure can be expressed by an \fotwo{} sentence in SNF,
\DoubleColumn
\begin{align*}
  \sentence_G &:= \left(\forall x\forall y: (E(x,y) \Rightarrow E(y,x))\land \neg E(x,x)\right)\land\\
  &\qquad \left(\forall x\exists y: E(x,y)\right),
\end{align*}
\DoubleColumnEnd
\SingleColumn
\begin{equation*}
  \sentence_G := \left(\forall x\forall y: (E(x,y) \Rightarrow E(y,x))\land \neg E(x,x)\right)\land \left(\forall x\exists y: E(x,y)\right),
\end{equation*}
\SingleColumnEnd
and the sampling problem corresponds to a WFOMS on $\sentence_G$ under $\weight(E) = \negweight(E) = 1$ over a domain of vertices $V = \{v_i\}_{i\in[n]}$.
% We fix the natural order of vertices, and say $v_j$ a predecessor of $v_i$ for if $j<i$.
In this sentence, the only realizable 1-type is $\neg E(x,x)$, and the realizable 2-tables are $\pi^1(x,y) = E(x,y)\land E(y,x)$ and $\pi^2(x,y) = \neg E(x,y)\land \neg E(y,x)$ representing the connectedness of two vertices.

\begin{figure}
  \DoubleColumn
  \centerline{\includegraphics[width=.47\textwidth]{figs/non-isolated_graph.png}}
  \DoubleColumnEnd
  \SingleColumn
  \centerline{\includegraphics[width=.67\textwidth]{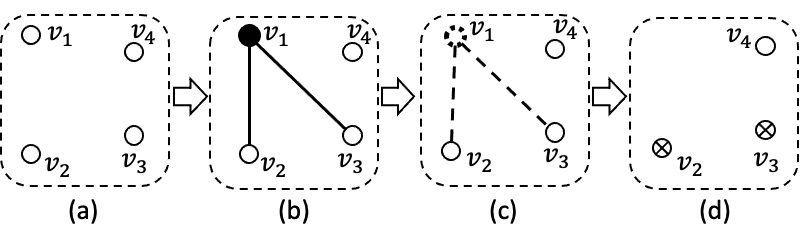}}
  \SingleColumnEnd
  \caption{A sampling step for an undirected graph with no isolated vertices: (a) begin with an initial graph that has no edges, and in the more general sampling problem, $V_\forall = V_\exists = \{v_1,v_2,v_3,v_4\}$; (b) sample edges for the vertex $v_1$; (c) remove the vertex $v_1$ with its sampled edges; (d) and obtain a graph with some vertices already non-isolated ($v_2$ and $v_3$), resulting in a new sampling problem with $V_\forall' = \{v_2,v_3,v_4\}$ and $V_\exists' = \{v_4\}$.}
  \label{fig:non-isolated_graph}
\end{figure}

We first apply the following transformation
on $\sentence_G$ resulting in $\sentence_{GT}$:
\begin{enumerate}
  \item introduce an auxiliary \textit{Tseitin predicate} $Z/1$ that indicates the non-isolation of vertices, and append $\forall x: Z(x)\Leftrightarrow \exists y: E(x,y)$ to $\sentence_G$, and
  \item remove $\forall x\exists y: E(x,y)$,
\end{enumerate}
and set the weight of the predicate $Z$ to $\weight(Z)=\negweight(Z) = 1$.
Then we consider a bit more general WFOMS of $(\sentence_{GT}\land\bigwedge_{v\in V_\exists}Z(v), V_\forall, \weight, \negweight)$, where $V_\exists\subseteq V_\forall\subseteq V$ and $V_\exists$ represents the set of vertices that should be non-isolated in the graph induced by $V_\forall$.
The original WFOMS on $\sentence_G$ can be clearly reduced to the more general problem by setting $V_\exists=V_\forall=V$,
and the reduction is sound with the mapping function $f(\mu') = \proj{\mu'}{\{E\}}$.

Let $\generalsentence_G$ denote the sentence $\sentence_{GT}\land\bigwedge_{v\in V_\exists} Z(v)$.
For the more general WFOMS, 
since $\{E\}$ is its skeleton, the problem is equivalent to sampling an $\{E\}$-structure $\structure$ over $V_\forall$ according to the probability $\pro[\structure\mid \generalsentence_G]$.
Given an $\{E\}$-structure $\structure$ over $V_\forall$, denote by $\structure_i$ the \textit{substructure} of $\structure$ concerning the vertex $v_i\in V_\forall$, which consists of the 2-tables of all vertex tuples containing $v_i$ and other vertices in $V_\forall$:
\begin{equation*}
  \structure_i := \bigcup_{v_j\in V_\forall: j\neq i} \pi_{i,j}(v_i, v_j),
\end{equation*}
where $\pi_{i,j}$ is the 2-table of $(v_i, v_j)$.
Following the domain recursion scheme, we choose a vertex $v_t$ from $V_\forall$ and decompose the sampling probability of $\structure$ into
\begin{equation*}
  \pro[\structure\mid \generalsentence_G] 
  =\pro\left[\structure\mid\generalsentence_G\land\structure_t\right]\cdot\pro[\structure_t\mid\generalsentence_G].
\end{equation*}
The decomposition leads to two successive subproblems of the general WFOMS: the first one is to sample a valid substructure $\structure_t$ from $\generalsentence_G$; the other can be viewed as a new WFOMS on $\generalsentence_G$ given the sampled $\structure_t$.

We first show that the new WFOMS can be also reduced to the general WFOMS but with the smaller domain
$V_\forall' = V_\forall \setminus \{v_t\}$ and 
\begin{equation*}
  V_\exists' = \{v_i\mid v_i\in V_\exists: v_i\neq v_t \land \pi_{t,i} = \pi^2 \}.
  % \label{eq:v_exists}
\end{equation*}
The reduction is obviously sound because every model of the WFOMS $(\sentence_{GT}\land\bigwedge_{v\in V_\exists'}Z(v), V_\forall', \weight, \negweight)$ can be mapped to a unique model of the WFOMS $(\generalsentence_G\land\structure_t, V_\forall, \weight, \negweight)$, and vice versa, without affecting the weight of the models.
Thus, the decomposition can be performed recursively on any WFOMS on $\generalsentence_G$ over $V_\forall$.
% Then, for any WFOMS on $\generalsentence_G$ over $V_\forall$, the sampling algorithm can be performed in a recursive manner.
Specifically, the algorithm takes $V_\forall$ and $V_\exists$ as input, 
\begin{enumerate}
  \item selects a vertex $v_t$ from $V_\forall$, 
  \item samples its substructure $\structure_t$ according to the probability $\pro[\structure_t\mid\generalsentence_G]$, and
  \item obtains a new problem with updated $V_\forall'$ and $V_\exists'$ for recursion.
\end{enumerate}
The recursion procedure terminates when all substructures $\structure_i$ are sampled ($V_\forall$ contains a single vertex), or the problem degenerates to a WFOMS on \ufotwo{} sentence ($V_\exists$ is empty).
The number of recursions is less than $|V|$, the total number of vertices.
% \footnote{The algorithm can also terminate when $V_\exists$ is empty and the sampling problem degenerates to the trivial uniform generation of undirected graphs.}
An example of a recursion step is shown in Figure~\ref{fig:non-isolated_graph}.

The remaining problem is to sample the substructure $\structure_t$ according to $\pro[\structure_t\mid\generalsentence_G]$.
Recall that $\structure_t$ determines the edges between $v_t$ and vertices in $V_\forall'$.
Let $V_1 = V_\forall' \setminus V_\exists$ and $V_2 = V_\forall' \setminus V_1$.
The sampling of $\structure_t$ can be realized by performing two random binary partitions on $V_1$ and $V_2$ respectively, resulting in $\{V_{11}, V_{12}\}$ and $\{V_{21}, V_{22}\}$, where the vertices in $V_{11}$ and $V_{21}$ will be connected to $v_t$, while the vertices in $V_{12}$ and $V_{22}$ will not be connected to it.
It can be demonstrated that the sampling probability of a substructure $\structure_t$ only depends on the size $(|V_{11}|, |V_{12}|, |V_{21}|, |V_{22}|)$.
% of its corresponding partitions.
The proof of this claim can be found in Section~\ref{sub:sampling_algorithm}, where the more general case of \fotwo{} sampling is addressed.
As a result, the sampling of $\structure_t$ can be accomplished by a random generation of the partition size, followed by two random partitions of the sampled size on $V_1$ and $V_2$ respectively.
We use the enumerative sampling method to generate a partition size.
The number of all possible partition sizes is clearly polynomial in $|V_\forall|$, and it will be shown in Section~\ref{sub:sampling_algorithm} that the sampling probability of each partition size can be computed in time polynomial in $|V_\forall'|$. 
Therefore, the complexity of the sampling algorithm is polynomial in the number of vertices.
This, together with the complexity of the recursion procedure, which is also polynomial in the number of vertices, implies that the whole sampling algorithm is lifted.

\subsection{A More General Sampling Problem}
\label{sub:generalproblem}

W.l.o.g.\footnote{
Any SNF sentence can be transformed into such form by introducing an auxiliary predicate $R_k$ with weights $\weight(R_k)=\negweight(R_k)=1$ for each $\extformula_k(x,y)$, append $\forall x\forall y: R_k(x,y) \Leftrightarrow \extformula_k(x,y)$ to the sentence, and replacing $\extformula_k(x,y)$ with $R_k(x,y)$. The transformation is obviously sound when viewing it as a reduction in WFOMS.}
, we suppose that each formula $\extformula_k(x,y)$ in the SNF sentence \eqref{eq:scott_form} is an atomic formula $R_k(x,y)$, where $R_k$ is a binary predicate in $\mathcal{P}_{\fotwoformula(x,y)}$, and its weights $\weight(R_k) = \negweight(R_k) = 1$.
We first construct the following sentence from the SNF one:
\DoubleColumn
\begin{equation}
  \begin{aligned}
  \sentence_T := &\forall x\forall y: \fotwoformula(x,y) \\
  &\land \bigwedge_{k\in [m]} \forall x: Z_k(x) \Leftrightarrow (\exists y: R_k(x,y)),
  \end{aligned}
  \label{eq:tseitin_reduction}
\end{equation}
\DoubleColumnEnd
\SingleColumn
\begin{equation}
  \sentence_T := \forall x\forall y: \fotwoformula(x,y) \land \bigwedge_{k\in [m]} \forall x: Z_k(x) \Leftrightarrow (\exists y: R_k(x,y)),
  \label{eq:tseitin_reduction}
\end{equation}
\SingleColumnEnd
where each $Z_k/1$ is a Tseitin predicate with the weight $\weight(Z_k) = \negweight(Z_k) = 1$.
% Note that if we append the conjunction over all Tseitin atoms $\bigwedge_{i\in [n]} \bigwedge_{j\in [m]} Z_j(e_i)$ to $\sentence_T$ and view it as a reduction on the WFOMS of $\sentence$, the reduction is sound.
% And if we see $\sentence_T$ as a reduction on the WFOMS of the \ufotwo{} sentence $\forall x\forall y:\fotwoformula(x,y)$, the reduction is also sound.
% Thus the evidences $Z_i(e_i)$'s characterize the sampling problem, bridging the WFOMS between \fotwo{} and \ufotwo{} fragments.

We then consider a more general WFOMS problem on the following sentence
% \ledit{
\begin{equation}
  \generalsentence := \sentence_T \land \bigwedge_{i\in [n]} \mathcal{C}_i
  \label{eq:general_sentence}
\end{equation}
over a domain of $\{e_i\}_{i\in[n]}$,
where each $\mathcal{C}_i$ is a conjunction over a subset of the ground atoms $\{Z_k(e_i)\}_{k\in[m]}$.
% \lnote{I used $\beta_i(x)$ here in the same spirit of 1-types and 2-tables.. I called it block type in the next subsection.. I can change the index $i$ for existential quantified formulas to some specific symbol.}
We call $\mathcal{C}_i$ the \textit{existential constraint} on the element $e_i$ and allow $\mathcal{C}_i = \top$, which means $e_i$ is not existentially quantified.
% }
% \ledit{
The more general WFOMS can be regarded as a conditional sampling problem, where the existential constraint serves as unary facts that condition the problem.
% }
The original WFOMS on $\sentence$ can be reduced to a more general problem by setting all existential constraints to be $\bigwedge_{k\in [m]} Z_k(x)$.
On the other hand, the WFOMS on the \ufotwo{} sentence $\forall x\forall y: \fotwoformula(x,y)$ is also reducible to the problem with $\mathcal{C}_i = \top$ for all $i\in[n]$. 
% On the other hand, if we set all $\beta_i(x)$'s to be $\top$, the problem is reducible to the WFOMS on the \ufotwo{} sentence $\forall x\forall y: \fotwoformula(x,y)$, which has been proved to be liftable under sampling in the previous literature~\cite{wangDomainLiftedSamplingUniversal2022}.
It is easy to check that these two reductions are both sound.
The main idea of our sampling algorithm is to use the domain recursion scheme to \textit{gradually remove the existential constraints until we eventually end up with a WFOMS problem on a \ufotwo{} sentence}.

\subsection{Partitioning the Domain}

Unless stated otherwise, in the rest of this section, 1-types and 2-tables are defined over  $\mathcal{P}_\sentence$, where $\sentence$ is the sentence in SNF.
Please bear in mind that the Tseitin predicates $Z_k$ are not in these 1-types.

We introduce the concepts of block and cell types as extensions of 1-types. These types are utilized in a manner akin to 1-types in the sampling algorithm for \ufotwo{}.
% }
% \ledit{We next define the notions of \textit{blocks} and \textit{cells} for partitioning the domain. }
% \ledit{
Consider a sentence $\generalsentence$ of the form \eqref{eq:general_sentence} with Tseitin predicates $Z_k$.
% }
A \textit{block type} $\beta$ is a subset of the atoms $\{Z_k(x)\}_{k\in[m]}$.
The number of the block types is $2^m$, where $m$ is the number of existentially-quantified formulas.
We often represent a block type as $\beta(x)$ and view it as a conjunctive formula over the atoms within the block.
% \ledit{
With the notion of block type, we can write $\generalsentence = \sentence_T\land\bigwedge_{i\in[n]} \beta_i(e_i)$, where the grounding $\beta_i(e_i)$ is exactly the existential constraint $\mathcal{C}_i$ imposed on $e_i$.
We call $\beta_i$ the block type of $e_i$.
% }
We fix the order of all block types and denote by $\beta^i$ the $i$th block type.
The domain $\domain$ is then partitioned by the blocks $\{B_{\beta^i}\}_{i\in[2^m]}$,
where each subset $B_{\beta^i} \subseteq \domain$ contains precisely the domain elements with the block type $\beta^i$.
It is important to note that the block types only indicate which Tseitin atoms should hold for a given element, and the Tseitin atoms not covered by the block types are left unspecified.
In contrast, the 1-types explicitly determine the truth values of all unary and reflexive atoms, excluding the Tseitin atoms.

The blocks are further partitioned into \textit{cells}.
A \textit{cell type} $\eta = (\beta, \tau)$ is a pair of a block type $\beta$ and a 1-type $\tau$.
We also write a cell type as a conjunctive formula of $\eta(x) = \beta(x) \land \tau(x)$.
Given a $\mathcal{P}_\sentence$-structure $\structure$, the \textit{cell type of an element $e$} is the combination of its block type (which is given by the sentence $\generalsentence$) and its realizing 1-type in $\structure$.
Each block $B_{\beta}$ is partitioned by the \textit{cells} $\{C^\structure_\eta \mid \eta=(\beta, \tau), \tau\in U_{\mathcal{P}_\sentence}\}$,
where each cell $C^\structure_\eta \subseteq B_\beta$ contains precisely the domain elements that are of cell type $\eta$.

For brevity, we denote by $N_u = |U_{\mathcal{P}_\sentence}|$, the number of all 1-types and $N_c = 2^m \times N_u$, the number of all cell types.
We fix a linear order of 1-types as well as cell types, and let $\tau^i$ and $\eta^j$ be the $i$th 1-type and $j$th cell type respectively.
Given a $\mathcal{P}_\generalsentence$-structure $\structure$ with the cell partition $\{C_{\eta^i}^\structure\}_{i\in[N_c]}$, 
we call the size $\left(|C^\structure_{\eta^i}|\right)_{i\in[N_c]}$ of the cell partition the \textit{cell configuration} of $\structure$, and $\left(|C^\structure_{(\beta, \tau^i)}|\right)_{i\in[N_u]}$ the \textit{cell configuration of $\structure$ in the block $\beta$}.
A $\mathcal{P}_\generalsentence$-structure will have a unique cell configuration (in a block).

We will often care about the set of all cell configurations over a set of elements, which is defined as the \textit{configuration space}. 
\begin{definition}[\textbf{Configuration Space}]
  Given a nonnegative integer $M$ and a positive integer $m$, we define the configuration space $\mathcal{T}_{M, m}$ as
  \begin{equation*}
    \mathcal{T}_{M,m} = \left\{(n_i)_{i\in[m]}\mid \sum_{i\in[m]} n_i = M, n_1, n_2, \dots, n_m\in\nat\right\}.
  \end{equation*}
\end{definition}
The size of $\mathcal{T}_{M, m}$ is given by $\binom{M + m - 1}{m - 1}$, which is clearly polynomial in $M$ (while exponential in $m$).

\subsection{The Sampling Algorithm}
\label{sub:sampling_algorithm}

We now describe our algorithm for the WFOMS of $(\generalsentence, \domain, \weight, \negweight)$ where $\generalsentence$ is of the form \eqref{eq:general_sentence} and $\domain = \{e_i\}_{i\in[n]}$, and prove that the algorithm is domain-lifted (i.e. runs in time polynomial in the domain size).
It can be easily verified that $\mathcal{P}_\sentence$ is a skeleton of the WFOMS problem.
Hence, this WFOMS problem is equivalent to sampling a valid $\mathcal{P}_\sentence$-structure $\structure$ according to the probability $\pro[\structure\mid\generalsentence]$.

Given a $\mathcal{P}_\sentence$-structure $\structure$ over $\domain$, let $\tau_i$ be the 1-type of the element $e_i$, and denote by $\eta_i = (\beta_i, \tau_i)$ the cell type of $e_i$.
% and $\pi_{i,j}$ the 2-table of $(e_i, e_j)$.
% Denote by $\structure_i$ the union of ground 2-tables of all element tuples consisting of the element $e_i$:
% \begin{equation*}
%   \structure_i := \bigcup_{j\in[n]: j\neq i} \pi_{i,j}(e_i, e_j).
% \end{equation*}
Using the notation of conditional probability, we decompose the sampling probability as
\begin{align*}
    \pro[\structure\mid \generalsentence]=\pro\left[\bigwedge_{i\in[n]}\tau_i(e_i)\mid \generalsentence\right] \cdot \pro\left[\structure \mid \generalsentence\land\bigwedge_{i\in[n]}\tau_i(e_i)\right]
\end{align*}
% \begin{equation}
%   \begin{aligned}
%      &\pro\left[\structure \mid \bigwedge_{i\in[n]}\eta_i(e_i)\land \generalsentence\right]\cdot \pro\left[\bigwedge_{i\in[n]}\eta_i(e_i)\mid \generalsentence\right] = \\
%      &\qquad \pro\left[\structure \mid \bigwedge_{i\in[n]}\eta_i(e_i)\land \sentence_T\right]\cdot \pro\left[\bigwedge_{i\in[n]}\eta_i(e_i)\mid \generalsentence\right]
%   \end{aligned}
%   \label{eq:sampling_cell_type}
% \end{equation}
Following a similar idea to the one used for sampling models from \ufotwo{} sentences, our proposed algorithm is divided into two phases---we first sample the 1-type for each element, and then handle the sampling of the structure conditional on the sampled 1-types.

\subsubsection{Sampling 1-Types}

\label{sub:sampling_cell_types}

% We first consider the sampling of $\eta_i$'s according to the probability $\pro\left[\bigwedge_{i\in[n]}\eta_i(e_i)\mid\generalsentence\right]$.
% \lnote{I slightly change the sampling problem from cell types to 1-types. I'm not sure if the presentation is more clear now.}
Let us first consider the sampling of 1-types according to $\pro[\bigwedge_{i\in[n]}\tau_i(e_i)\mid\generalsentence]$.
Note that the block type of each element has been determined by the sentence $\generalsentence$, and thus the problem is equivalent to sampling the cell type of each element~\footnote{In the special case where there is a single block, e.g., $\generalsentence$ is of the standard SNF or in \ufotwo{}, the sampling problem can be simplified.}.
% which can be achieved by randomly partitioning each block into cells.
% The sampling probability $\pro\left[\bigwedge_{i\in[n]}\tau_i(e_i)\mid\generalsentence\right]$ can be written as
% \begin{equation}
%   \frac{\symwfomc(\sentence_T\land \bigwedge_{i\in[n]}\eta_i(e_i),\domain,\weight, \negweight)}{\symwfomc(\generalsentence, \domain, \weight, \negweight)}
%   \label{eq:sampling_cell_type}
% \end{equation}
% Sampling the 1-types 
% % according to the probability $\pro\left[\bigwedge_{i\in[n]}\tau_i(e_i)\mid\generalsentence\right]$
% is equivalent to sampling the cell types,
% %  with the unnormalized weight $\symwfomc(\sentence_T\land \bigwedge_{i\in[n]}\eta_i(e_i),\domain,\weight, \negweight)$
% which can be achieved by randomly partitioning each block to cells.
% % The block type of each element is determined by the sentence $\generalsentence$ itself, and the remaining problem is to further partition each block into cells.
Due to the symmetry of the WFOMS problem, the sampling probability of a cell partition is completely determined by its corresponding cell configuration. 
Therefore, the problem of randomly partitioning cells is further reduced to a problem of sampling cell configurations. 
% This is similar to the process of reducing the sampling of 1-types to the sampling of 1-types cardinality in the proof of Theorem~\ref{th:ufo_liftable}.

The sampling algorithm for cell configurations is outlined in Algorithm~\ref{alg:sampling_cell_types}. 
The algorithm begins by sampling a random cell configuration in Line~1-14, which is then used to randomly partition each block into cells in Line~15-21. 
While the partitioning process is relatively straightforward, in the following discussion, we will focus specifically on how to sample a cell configuration.

\begin{algorithm}[!htb] 
  \caption{\code{OneTypeSampler}($\generalsentence, \domain, \weight, \negweight$)} 
  \label{alg:sampling_cell_types}
  % \SetKwProg{generate}{Function \emph{generate}}{}{end}
  % \begin{flushleft}
  %   \textbf{INPUT:} A sentence $\generalsentence$ of the form \eqref{eq:general_sentence}, a domain $\domain$ of size $n$ and a weighting $(\weight, \negweight)$\\
  %   \textbf{OUTPUT:} $\eta_1,\eta_2,\dots,\eta_n$\\
  % \end{flushleft}
  \begin{algorithmic}[1]
  \State $W\gets \symwfomc(\generalsentence, \domain, \weight, \negweight)$
  \State Obtain the blocks $B_{\beta^1}, B_{\beta^2}, \dots, B_{\beta^{2^m}}$ from $\generalsentence$
%   \State // $\code{Prod}(\dots)$ is the Cartesian product
  \For{$\left(\vecn_{\beta^i}\right)_{i\in[2^m]}\in\code{Prod}\left(\mathcal{T}_{|B_{\beta^1}|, N_u}, \dots, \mathcal{T}_{|B_{\beta^{2^m}}|, N_u}\right)$}
    \State $\vecn \gets \bigoplus_{i\in[2^m]} \vecn_{\beta^i}$
    % \State // Compute the weighted sum over all models with the cell configuration
    \State Compute $\sumweight_{\vecn}$ by \eqref{eq:weight_cell_configuration}
    % \State // $\binom{|B_{\beta^t}|}{\vecn_{\beta^t}}$'s are multinomial coefficients
    \State $W' \gets \sumweight_{\vecn} \cdot \prod_{t=1}^{2^m}\binom{|B_{\beta^t}|}{\vecn_{\beta^t}}$
    \State // $\code{Uniform}(0,1)$ produces a uniformly random number over $[0,1]$
    \If{$\code{Uniform}(0, 1) < \frac{W'}{W}$}
      \State $\vecn^* \gets \vecn$
      \State \textbf{break}
    \Else
      \State $W \gets W - W'$
    \EndIf
  \EndFor
  \For{$i\in[2^m]$}
    \State Fetch the cell configuration $\vecn^*_{\beta^i}$ in $\beta^i$ from $\vecn^*$
    \State Randomly partition $B_{\beta^i}$ into $\{C_{\beta^i, \tau^j}\}_{j\in[N_u]}$ according to $\vecn^*_{\beta^i}$
    \For{$j\in[N_u]$}
      \State Assign the 1-type $\tau^j$ to all elements in $C_{\beta^i, \tau^j}$ 
    \EndFor
  \EndFor
  \end{algorithmic}
\end{algorithm}

The basic idea is again based on enumerative sampling.
Let $B_{\beta^1}, B_{\beta^2}, \dots, B_{\beta^{2^m}}$ be the blocks defined by $\generalsentence$.
Any cell configuration $\vecn$ can be viewed as a concatenation of $2^m$ cell configurations $\vecn_{\beta^1}, \vecn_{\beta^2},\dots,\vecn_{\beta^{2^m}}$ in the blocks, and each $\vecn_{\beta^i}$ is from the configuration space $\mathcal{T}_{|B_{\beta^i}|, N_u}$. 
Hence, in the algorithm, the enumeration of all cell configurations is performed by applying the Cartesian product function $\code{Prod}$ on the configuration spaces $\mathcal{T}_{|B_{\beta^1}|, N_u}, \mathcal{T}_{|B_{\beta^2}|, N_u},\dots, \mathcal{T}_{|B_{\beta^{2^m}}|, N_u}$.
% contains precisely all possible cell configurations in the block $B_{\beta^i}$.

For the computation of the sampling probability, we first observe that any cell partitions that produce the same configuration $\vecn$ will have the same sampling probability.
We use $\mathcal{W}_\vecn$ to denote the sampling weight (the numerator \wfomc{} of the probability) of any such cell partition.
Then the sampling probability of a cell configuration $\vecn$ can be derived from $\mathcal{W}_\vecn\cdot \prod_{i\in[2^m]}\binom{|B_{\beta^i}|}{\vecn_{\beta^i}}$, where the latter product equates to the total number of the partitions giving rise to $\vecn$.

The value of $\mathcal{W}_\vecn$ will play a crucial role in the remaining sampling algorithm.
In this context, we provide its formal definition.
Given a nonnegative integer vector $\vecn$ of size $N_c$, let $\widetilde{n} = \sum_{i\in[N_c]} n_i$, and $\mathcal{W}_\vecn$ be defined as
\begin{equation}
  \mathcal{W}_\vecn := \symwfomc(\sentence_T \land\bigwedge_{i\in[\widetilde{n}]}\widetilde{\eta_i}(\widetilde{e}_i), \widetilde{\domain}, \weight, \negweight),
  \label{eq:weight_cell_configuration}
\end{equation}
where $\{\eta_i\}_{i\in[\widetilde{n}]}$ is a set of cell types that gives rise to the configuration $\vecn$, $\widetilde{\domain}$ is a domain of size $\widetilde{n}$, and $\widetilde{e}_i$ is the $i$th element in $\widetilde{\domain}$.
Recall that a cell type can be represented by a conjunction of unary atoms, and thus each $\widetilde{\eta}_i(\widetilde{e}_i)$ in $\mathcal{W}_\vecn$ can be interpreted as a set of unary facts.
The value of $\sumweight_{\vecn}$ is then equal to the \wfomc{} of the \fotwo{} sentence $\sentence_T$ conditional on a set of unary facts.
Such conditional counting problems have been thoroughly studied by~\citet{vandenbroeckConditioningFirstorderKnowledge2012}, and the computational complexity has been shown to be polynomial in the domain size $\widetilde{n}$ and the number of facts.
Since the number of facts is clearly polynomial in $\widetilde{n}$, computing $\mathcal{W}_\vecn$ can be done in time polynomial in $\widetilde{n}$.

\begin{lemma}
  The complexity of $\code{OneTypeSampler}(\cdot, \cdot, \cdot, \cdot)$ in Algorithm~\ref{alg:sampling_cell_types} is polynomial in the size of the input domain.
  \label{lemma:complexity_cell_type}
\end{lemma}
\begin{proof}
  For each block $B_{\beta^i}$, there are totally $|\mathcal{T}_{|B_{\beta^i}|, N_u}|$ possible cell configurations in $B_{\beta^i}$.
  Traversing over all blocks in the loop at line~3 in Algorithm~\ref{alg:sampling_cell_types} will enumerate a total of $\prod_{i=1}^{2^m} |\mathcal{T}_{|B_{\beta^i}|, N_u}|$ possible cell configurations.
  Even though this number may appear daunting, it is polynomial in the domain size by the definition of configuration space.
  The remaining complexity of the algorithm is derived from the computation of $\sumweight_{\vecn}$, which has been shown to be polynomial-time in $\sum_{i\in [N_c]} n_i$.
  Finally, the value of $\sum_{i\in [N_c]} n_i$ is equal to the domain size, completing the proof.
\end{proof}

\subsubsection{Domain Recursive Sampling}

Now, let us consider the sampling problem of $\mathcal{P}_\sentence$-structures conditional on the sampled 1-types $\eta_i$.
We rewrite the sampling probability $\pro[\structure\mid\generalsentence\land\bigwedge_{i\in[n]}\tau_i(e_i)]$ as $\pro[\structure\mid\sentence_T\land\bigwedge_{i\in[n]}\eta_i(e_i)]$ for better presentation.

Let $\pi_{i,j}$ be the 2-table of the element tuple $(e_i, e_j)$, 
and denote by $\structure_i$ the set of ground 2-tables over all element tuples involved in the element $e_i$:
\begin{equation*}
  \structure_i := \bigcup_{j\in[n]: j\neq i} \pi_{i,j}(e_i, e_j).
\end{equation*}
Following the idea of domain recursion, we select an element $e_t$ from $\domain$ and decompose the sampling probability as
\begin{align*}
  \pro\left[\structure\mid \sentence_T\land \bigwedge_{i\in[n]}\eta_i(e_i)\right]=\pro\left[\structure \mid \sentence_T\land\bigwedge_{i\in[n]}\eta_i(e_i)\land\structure_t \right]\cdot \pro\left[\structure_t\mid\sentence_T\land\bigwedge_{i\in[n]}\eta_i(e_i)\right].
\end{align*}
We first demonstrate that for any valid substructure $\structure_t$ of the sampling problem, the WFOMS specified by the probability $\pro[\structure\mid\sentence_T\land\bigwedge_{i\in[n]}\eta_i(e_i)\land\structure_t]$ can be reduced to a WFOMS of the same form as the original problem on $\sentence_T\land\bigwedge_{i\in[n]}\eta_i(e_i)$, but over a smaller domain $\domain' = \domain \setminus {e_t}$.
% This decomposition has a convenient form that allows for the recursive sampling of $\structure_i$'s. Specifically, we will demonstrate that for any valid substructure $\structure_t$ in the sampling problem $(\bigwedge_{i\in[n]}\eta_i(e_i)\land\sentence_T, \domain,\weight, \negweight)$, the WFOMS specified by the probability in Equation~\eqref{eq:domain_recursive} can be reduced to a WFOMS of the same form as the original problem on $\bigwedge_{i\in[n]}\eta_i(e_i)\land\sentence_T$, but over a smaller domain $\domain' = \domain \setminus {e_t}$.

Given a 2-table $\pi$ and a block type $\beta$, let $\relaxcelltype{\beta}{\pi}$ be a new block type:
\begin{equation*}
  \relaxcelltype{\beta}{\pi} = \beta \setminus \{Z_k(x)\mid k\in[m]: R_k(y,x)\in\pi(x,y)\}
  % \label{eq:relax_block_type}
\end{equation*}
% \lnote{Find a more concise notation for relaxation.}
We call $\relaxcelltype{\beta}{\pi}$ the \textit{relaxed} block type of $\beta$ under $\pi$, as it removes a part of the existential constraint that is already satisfied by the relations in $\pi$.
We can also apply the relaxation under $\pi$ on a cell type $\eta = (\beta, \tau)$, resulting in $\relaxcelltype{\eta}{\pi} = (\relaxcelltype{\beta}{\pi}, \tau)$.
% Recall that the substructure $\structure_t$ consists of all 2-tables $\pi_{n,1}, \pi_{n,2},  \dots, \pi_{n, n-1}$ concerning the element $e_n$.
% for all $i\in[n-1]$, $\eta_{i\mid\pi_{n,i}}$ presents the relaxed cell type of the element $e_i$.
Let 
\begin{equation}
  \recursivesentence = \sentence_T\land\bigwedge_{i\in[n]}\eta_i(e_i)\land \structure_t,
  \label{eq:recursive_sentence}
\end{equation}
and
\begin{equation}
  \recursivesentence' = \sentence_T\land \bigwedge_{i\in[n]\setminus \{t\}}\relaxcelltype{\eta_i}{\pi_{t,i}}(e_i).
  \label{eq:reduced_sentence}
\end{equation}
We have the following lemma.
\begin{lemma}
  \label{lemma:modular}
  If $\structure_t$ is valid w.r.t. the WFOMS of $(\sentence_T\land\bigwedge_{i\in[n]}\eta_i(e_i),\domain,\weight, \negweight)$, i.e., $\recursivesentence$ is satisfiable, the reduction from the WFOMS of $(\recursivesentence, \domain, \weight, \negweight)$ to $(\recursivesentence', \domain', \weight, \negweight)$ is sound.
\end{lemma}

\begin{proof}
  Let $\mathfrak{S}_1$ and $\mathfrak{S}_2$ be the WFOMS of $(\recursivesentence, \domain, \weight, \negweight)$ and $(\recursivesentence', \weight, \domain', \negweight)$ respectively.
  Both of $\mathfrak{S}_1$ and $\mathfrak{S}_2$ have $\mathcal{P}_\sentence$ as their skeleton.
  Thus, the problem of $\mathfrak{S}_1$ (resp. $\mathfrak{S}_2$) is equivalent to sampling a $\mathcal{P}_\sentence$-structure over the domain $\domain$ (resp. $\domain'$).
  The mapping function in the reduction can be defined on $\mathcal{P}_\sentence$-structures over $\domain'$.
  We argue that the mapping function is $f(\structure') = \structure'\cup \structure_t\cup \tau_t(e_t)$.
  The function $f$ is clearly deterministic and polynomial-time.

  To simplify the rest arguments of the proof, we will first show that $f$ is bijective, i.e., 
  for any valid $\mathcal{P}_\sentence$-structure $\structure$ of $\mathfrak{S}_1$, there exists a unique valid structure $\structure'$ of $\mathfrak{S}_2$ such that $f(\structure') = \structure$.
  Let the respective structure to be $\structure' = \structure \setminus \structure_t \setminus \{\tau_t(e_t)\}$, and the uniqueness is clear.
  % We can write $\structure$ as $\structure = (\structure_t \cup \tau_t(e_t)) \cup \structure_{-t}$, and the unique structure $\structure'$ is $\structure' = \bigcup_{i\in[n]: i\neq t}\tau_i(e_i)\cup \structure_{-t}$.
  Next, we demonstrate that $\structure'$ is valid w.r.t. $\mathfrak{S}_2$.
  Ground out $\recursivesentence$ into $\domain$:
  \begin{align*}
    \structure_t\land&\bigwedge_{i\in[n]}\eta_i(e_i)\land \bigwedge_{i,j\in[n]} \fotwoformula(e_i, e_j)\land \Lambda,
  \end{align*}
  where $\Lambda = \bigwedge_{k\in[m]}\bigwedge_{i\in[n]} \left(Z_k(e_i)\Leftrightarrow \bigvee_{j\in[n]} R_k(e_i, e_j)\right)$.
  By replacing the ground Tseitin atoms in cell types $\eta_i(e_i)$ in $\Lambda$ with their corresponding truth assignments and then discarding $\Lambda$, we obtain a ground formula without Tseitin atoms:
  \DoubleColumn
  \begin{equation}
    \begin{aligned}
      \structure_t\land\bigwedge_{i\in[n]}\tau_i(e_i)&\land \bigwedge_{i,j\in[n]} \fotwoformula(e_i,e_j)\\
      &\quad \land \bigwedge_{i\in[n]}\bigwedge_{\substack{k\in[m]: \\Z_k(x)\in\beta_i}} \bigvee_{j\in[n]} R_k(e_i, e_j).
    \end{aligned}
    \label{eq:ground}
  \end{equation}
  \DoubleColumnEnd
  \SingleColumn
  \begin{equation}
      \structure_t\land\bigwedge_{i\in[n]}\tau_i(e_i)\land \bigwedge_{i,j\in[n]} \fotwoformula(e_i,e_j) \land \bigwedge_{i\in[n]}\bigwedge_{\substack{k\in[m]: \\Z_k(x)\in\beta_i}} \bigvee_{j\in[n]} R_k(e_i, e_j).
    \label{eq:ground}
  \end{equation}
  \SingleColumnEnd
  It can be easily shown that $\structure$ is valid w.r.t. $\mathfrak{S}_1$, iff $\structure$ satisfies the formula~\eqref{eq:ground}.
  It follows that the structure $\structure'$ satisfies the following formula
  \DoubleColumn
  \begin{equation}
    \begin{aligned}
      \bigwedge_{i\in[n]\setminus \{t\}}&\tau_i(e_i)\land \bigwedge_{i,j\in[n]\setminus \{t\}} \fotwoformula(e_i,e_j)\\
      &\quad \land \bigwedge_{i\in[n]\setminus \{t\}}\bigwedge_{\substack{k\in[m]: \\Z_k(x)\in\relaxcelltype{\beta_{i}}{\pi_{t,i}}}} \bigvee_{j\in[n]\setminus\{t\}} R_k(e_i, e_j),
    \end{aligned}
    \label{eq:recursive_ground}
  \end{equation}
  \DoubleColumnEnd
  \SingleColumn
  \begin{equation}
      \bigwedge_{i\in[n]\setminus \{t\}}\tau_i(e_i)\land \bigwedge_{i,j\in[n]\setminus \{t\}} \fotwoformula(e_i,e_j) \land \bigwedge_{i\in[n]\setminus \{t\}}\bigwedge_{\substack{k\in[m]: \\Z_k(x)\in\relaxcelltype{\beta_{i}}{\pi_{t,i}}}} \bigvee_{j\in[n]\setminus\{t\}} R_k(e_i, e_j),
    \label{eq:recursive_ground}
  \end{equation}
  \SingleColumnEnd
  which is obtained from \eqref{eq:ground} by substituting the ground atoms in $\structure_t$ and $\tau_t(e_t)$.
  The formula~\eqref{eq:recursive_ground} is nothing else but the grounding of $\recursivesentence'$ over $\domain'$ followed by the same replacement of ground Tseitin atoms in the relaxed cell types $\relaxcelltype{\eta_{i}}{\pi_{t,i}}$. 
  So we can conclude that $\structure'$ is also valid w.r.t. $\mathfrak{S}_2$.

  Now, we are prepared to demonstrate the consistency of sampling probability through the mapping function.
  Since $f$ is bijective, it remains to be shown that
  \begin{equation*}
    \pro[f(\structure') \mid\recursivesentence; \domain, \weight, \negweight] = \pro[\structure' \mid \recursivesentence'; \domain', \weight, \negweight]
  \end{equation*}
  for any valid structure $\structure'$ of $\mathfrak{S}_2$.
  By the definition of the mapping function $f$, we have
  \begin{equation*}
    \typeweight{f(\structure')} = \typeweight{\structure'}\cdot\typeweight{\structure_t}\cdot\typeweight{\tau_t}.
  \end{equation*}
  Moreover, due to the bijection of $f$ and the fact that $\mathcal{P}_\sentence$ is a skeleton of $\mathfrak{S}_1$ and $\mathfrak{S}_2$, we have
  \DoubleColumn
  \begin{equation}
    \begin{aligned}
      &\symwfomc(\recursivesentence, \domain, \weight, \negweight) =\sum_{\mu\in\fomodels{\recursivesentence}{\domain}}\typeweight{\proj{\mu}{\mathcal{P}_\sentence}}\\
      &=\sum_{\mu'\in\fomodels{\recursivesentence'}{\domain'}}\typeweight{f(\proj{\mu'}{\mathcal{P}_\sentence})}\\
      &=\typeweight{\structure_t}\cdot\typeweight{\tau_t}\cdot \sum_{\mu'\in\fomodels{\recursivesentence'}{\domain'}}\typeweight{\proj{\mu'}{\mathcal{P}_\sentence}}\\
      &=\typeweight{\structure_t}\cdot\typeweight{\tau_t}\cdot\symwfomc(\recursivesentence', \domain', \weight, \negweight).
    \end{aligned}
    \label{eq:modular_wfomc}
  \end{equation}
  \DoubleColumnEnd
  \SingleColumn
  \begin{equation}
    \begin{aligned}
      \symwfomc(\recursivesentence, \domain, \weight, \negweight) &=\sum_{\mu\in\fomodels{\recursivesentence}{\domain}}\typeweight{\proj{\mu}{\mathcal{P}_\sentence}}\\
      &=\sum_{\mu'\in\fomodels{\recursivesentence'}{\domain'}}\typeweight{f(\proj{\mu'}{\mathcal{P}_\sentence})}\\
      &=\typeweight{\structure_t}\cdot\typeweight{\tau_t}\cdot \sum_{\mu'\in\fomodels{\recursivesentence'}{\domain'}}\typeweight{\proj{\mu'}{\mathcal{P}_\sentence}}\\
      &=\typeweight{\structure_t}\cdot\typeweight{\tau_t}\cdot\symwfomc(\recursivesentence', \domain', \weight, \negweight).
    \end{aligned}
    \label{eq:modular_wfomc}
  \end{equation}
  \SingleColumnEnd
  Finally, by the definition of conditional probability, we can write
  \DoubleColumn
  \begin{equation}
    \begin{aligned}
      &\pro[f(\structure')\mid\recursivesentence;\domain, \weight, \negweight]=\frac{\typeweight{f(\structure')}}{\symwfomc(\recursivesentence, \domain, \weight, \negweight)}\\
      &=\frac{\typeweight{\structure'}\cdot\typeweight{\structure_t}\cdot\typeweight{\tau_t}}{\typeweight{\structure'}\cdot\typeweight{\structure_t}\cdot\symwfomc(\recursivesentence', \domain', \weight, \negweight)}\\
      &=\frac{\typeweight{\structure'}}{\symwfomc(\recursivesentence', \domain', \weight, \negweight)}\\
      &=\pro[\structure'\mid\recursivesentence';\domain', \weight, \negweight],
    \end{aligned}
    \label{eq:consistent_weight}
  \end{equation}
  \DoubleColumnEnd
  \SingleColumn
  \begin{equation}
    \begin{aligned}
      \pro[f(\structure')\mid\recursivesentence;\domain, \weight, \negweight]&=\frac{\typeweight{f(\structure')}}{\symwfomc(\recursivesentence, \domain, \weight, \negweight)}\\
      &=\frac{\typeweight{\structure'}\cdot\typeweight{\structure_t}\cdot\typeweight{\tau_t}}{\typeweight{\structure'}\cdot\typeweight{\structure_t}\cdot\symwfomc(\recursivesentence', \domain', \weight, \negweight)}\\
      &=\frac{\typeweight{\structure'}}{\symwfomc(\recursivesentence', \domain', \weight, \negweight)}\\
      &=\pro[\structure'\mid\recursivesentence';\domain', \weight, \negweight],
    \end{aligned}
    \label{eq:consistent_weight}
  \end{equation}
  \SingleColumnEnd
  and thus complete the proof.
\end{proof}

% \lnote{The rest of the section hasn't been revised.}
With the sound reduction presented above, what remains to the algorithm is the sampling of $\structure_t$ given the probability $\pro\left[\structure_t \mid \sentence_T\land \bigwedge_{i\in[n]} \eta_i(e_i) \right]$.

Recall that $\structure_t$ consists of the ground 2-tables of all tuples comprising $e_t$ and the elements in $\domain'$.
% , so sampling $\structure_t$ involves determining those 2-tables.
We follow a similar approach as in the cell type sampling and accomplish the sampling of $\structure_t$ through random partitions on cells.
Let ${C_{\eta^1}, C_{\eta^2}, \dots, C_{\eta^{N_c}}}$ be the cell partition of $\domain'$ corresponding to the sampled cell types $\eta_1, \eta_2, \dots, \eta_{n-1}$.
Let $N_b$ be the number of all 2-tables, and fix the linear order of 2-tables $\pi^1, \pi^2, \dots, \pi^{N_b}$.
Any substructure $\structure_t$ can be viewed as partitions on each cell into $N_b$ disjoint subsets; each subset corresponds to a 2-table $\pi^j$ and precisely contains the elements that realize $\pi^j$ in combination with $e_t$.

Given a substructure $\structure_t$, we use $\left\{G_{\eta^i, \pi^j}^{\structure_t}\right\}_{j\in[N_b]}$ to denote the refined partition on $C_{\eta^i}$, and $\vecg_{\eta^i}^{\structure_t} = \left(|G_{\eta^i, \pi^j}^{\structure_t}|\right)_{j\in[N_b]}$ its corresponding cardinality vector.
Let $\vecg^{\structure_t} = \bigoplus_{i\in[N_c]}\vecg_{\eta^i}^{\structure_t}$ be the concatenation of cardinality vectors over all cells, which is called the \textit{2-table configuration} of $\structure_t$.

We first assume that $\structure_t$ is valid in the sampling problem $(\sentence_T\land \bigwedge_{i\in[n]}\eta_i(e_i),\domain,\weight, \negweight)$.
It will turn out that the sampling probability of $\structure_t$ is completely determined by its corresponding 2-table configuration $\vecg^{\structure_t}$.
To begin with, as stated by~\eqref{eq:modular_wfomc}, we can write the sampling weight $\symwfomc(\sentence_T\land\bigwedge_{i\in[n]}\eta_i(e_i)\land\structure_t, \domain, \weight, \negweight)$ as
\begin{equation}
  \symwfomc(\recursivesentence', \domain', \weight, \negweight)\cdot \typeweight{\tau_t}\cdot \typeweight{\structure_t}
  \label{eq:sampling_a_n}
\end{equation}
where $\recursivesentence'$, defined as~\eqref{eq:reduced_sentence}, is the reduced sentence by the 2-tables in $\structure_t$.
Let $\vecn^{\structure_t}$ be the cell configuration corresponding to the ground cells in $\recursivesentence'$.
The value of $\symwfomc(\recursivesentence',\domain', \weight, \negweight)$ is exactly $\sumweight_{\vecn^{\structure_t}}$, which is formally defined in Section~\ref{sub:sampling_cell_types}.
Denote by $\vecweight = \left(\typeweight{\pi^i}\right)_{i\in N_b}$, the weight vector of 2-tables.
We can then write \eqref{eq:sampling_a_n} as
\begin{equation}
  \sumweight_{\vecn^{\structure_t}}\cdot \typeweight{\tau_t} \cdot \prod_{i\in[N_c]} \vecweight^{\vecg_{\eta^i}^{\structure_t}}.
  \label{eq:structure_weight}
\end{equation}
In the equation above, $\tau_t$ has already been decided in the cell type $\eta_t$ (by \code{OneTypeSampler}), and the last term only depends on the 2-table configurations $\vecg^{\structure_t}$.
It is easy to check that the cell configuration of $\vecn^{\structure_t}$ is also fully determined by $\vecg^{\structure_t}$.
% By the definition of relaxation on a cell type under a 2-table, for any cell type $\eta^i$ and any 2-table $\pi^j$, the reduction of $\recursivesentence'$ under $\structure_t$ moves $|G_{\eta^i, \pi^j}^{\structure_t}|$ domain elements from the cell $\eta^i$ to $\relaxcelltype{\eta^i}{\pi^j}$.
To illustrate this, let $n^{\structure_t}_{\eta}$ be the cardinality of cell type $\eta$ in $\vecn^{\structure_t}$,
and $g_{\eta, \pi}^{\structure_t}$ the cardinality of $\pi$ in $\vecg_{\eta}^{\structure_t}$, i.e., $|G_{\eta, \pi}^{\structure_t}|$.
For any cell type $\eta$, the value of $n^{\structure_t}_{\eta}$ be can computed by 
\begin{equation}
  n^{\structure_t}_{\eta} = \sum_{i\in[N_c], j\in[N_b]: \relaxcelltype{\eta^i}{\pi^j} = \eta} g_{\eta^i, \pi^j}^{\structure_t}.
  \label{eq:configuration_reduction}
\end{equation}

By the argument above, the sampling probability~\eqref{eq:structure_weight} of a valid substructure $\structure_t$ is completely determined by $\vecg^{\structure_t}$.
Thus, we can sample $\structure_t$, in the same spirit of sampling 1-types in Section~\ref{sub:sampling_cell_types}, by first sampling a 2-table configuration $\vecg^{\structure_t}$, and then partitioning the cells accordingly.
We can then simply apply the enumerative sampling method, as the number of possible 2-table configurations is clearly polynomial in the domain size.
For any 2-table configuration $\vecg^{\structure_t}$, its sampling weight can be computed by multiplying \eqref{eq:structure_weight} by $\prod_{i\in[N_c]}\binom{n_{\eta^i}}{\vecg_{\eta^i}^{\structure_t}}$, where $n_{\eta^i}$ is the size of $C_{\eta^i}$.

So far in our discussion, we have been always assuming that the substructure $\structure_t$ is valid in the WFOMS $(\sentence_T\land \bigwedge_{i\in[n]}\eta_i(e_i), \domain,\weight, \negweight)$, or it should not be sampled.
We guarantee this assumption by imposing some constraints on the 2-table configuration $\vecg^{\structure_t}$.
We call a 2-table $\pi$ \textit{coherent} with a 1-types tuple $(\tau, \tau')$ if, for some domain elements $a$ and $b$, the interpretation of $\tau(a)\cup\pi(a,b)\cup\tau'(b)$ satisfies the formula $\fotwoformula(a,b)\land\fotwoformula(b,a)$.
Then, the first constraint is that any 2-table $\pi_{t,i}$ in $\structure_t$ must be coherent with $\tau_t$ and $\tau_i$.
This translates to a requirement on 2-table configuration that, when partitioning a cell $\eta^i$, the cardinality of 2-tables that are not coherent with $\tau_t$ and $\tau_i$ is restricted to be $0$.
The second constraint is that, for any index $k\in\{i\mid Z_i\in\beta_t\}$, the substructure $\structure_t$ must contain at least one ground atom of the form $R_k(e_t,a)$, where $a$ is a domain element from $\domain$, to make $\structure_t$ satisfy the existential formula $\exists y:R_k(e_t,y)$.
This means that there must be at least one nonzero cardinality in the 2-table configuration such that its corresponding 2-table $\pi$ satisfies $R_k(x,y)\in\pi$.

By combining all the ingredients discussed above, we now present our sampling algorithm for the sentence $\sentence_T$ conditionally on the cell types $\eta_i$, as shown in Algorithm~\ref{alg:drsampler}.
The overall structure of the algorithm follows a recursive approach, where a recursive call  with a smaller domain and relaxed cell types is invoked at Line~33.
The algorithm terminates when the input domain contains a single element (at Line~1) or there are no existential constraints on the elements (at Line~4).
In Lines~10-23, all possible 2-table configurations are enumerated.
For each configuration, we compute its corresponding weight in Lines~13-15 and decide whether it should be sampled in Lines~16-21.
When the 2-table configuration has been sampled, we randomly partition the cells in Lines~25-32, and then update the sampled structure and the cell type of each element respectively at Line~29 and~30.
The function $\code{ExSat}(\vecg, \eta)$ at Line~12 is used to check whether the 2-table configuration $\vecg$ guarantees the validity of the sampled substructures, as discussed above.
The pseudo-code for this function is presented in Appendix~\ref{sub:exsat}.

\begin{algorithm}[!htb]
  \caption{$\code{DRSampler}(\sentence_T, \domain, \weight, \negweight, (\eta_i)_{i\in[n]})$} 
  \label{alg:drsampler}
  % \SetKwProg{generate}{Function \emph{generate}}{}{end}
  % \begin{flushleft}
  %   \textbf{INPUT:} A sentence $\sentence_T$ of the form \eqref{eq:general_sentence}, a domain $\domain=\{e_i\}_{i\in[n]}$, a weighting $(\weight, \negweight)$ and the cell type $\eta_i = (\tau_i,\beta_i)$ of each domain element $e_i$\\
  %   \textbf{OUTPUT:} A model $\mu$ of $\generalsentence$\\
  % \end{flushleft}
  \begin{algorithmic}[1]
    % \State $n\gets |\domain|$
    \If{$n = 1$}
    \State \Return $\emptyset$
    \EndIf
    \If{only the block of type $\beta =\top$ is nonempty}
      \State \Return a model $\mu$ sampled by a \ufotwo{} WMS from $\forall x\forall y: \fotwoformula(x,y)\land\bigwedge_{i\in[n]}\tau_i(e_i)$ over $\domain$ under $(\weight, \negweight)$
    \EndIf
    \State Choose $t\in[n]$; $\mu \gets \tau_t(e_t)$; $\domain'\gets\domain \setminus \{e_t\}$
    \State Get the cell configuration $\vecn = \left(n_{\eta^i}\right)_{i\in[N_c]}$ of $\left(\eta_i\right)_{i\in[n]\setminus\{t\}}$
    \State $W \gets \sumweight_{\vecn}$
    % \State $\vecweight \gets \left(\typeweight{\pi^i}\right)_{i\in N_b}$
    % \State $\forall i\in[N_c], T_i \gets \mathcal{T}_{n_{\eta^i}, N_b}$
    \For{$\left(\vecg_{\eta^i}\right)_{i\in[N_c]}\gets\code{Prod}(\mathcal{T}_{n_{\eta^1}, N_b}, \dots, \mathcal{T}_{n_{\eta^{N_c}}, N_b})$}
    \State $\vecg\gets\bigoplus_{i\in[N_c]}\vecg_{\eta^i}$
      \If{$\code{ExSat}\left(\vecg, \eta_t\right)$}
        \State Get the new cell configuration $\vecn'$ w.r.t $\vecg$ by \eqref{eq:configuration_reduction}
        \State Compute $\sumweight_{\vecn'}$ by~\eqref{eq:weight_cell_configuration}
        \State $W'\gets \sumweight_{\vecn'}\cdot \typeweight{\tau_t} \cdot \prod_{i\in[N_c]} \binom{n_{\eta^i}}{\vecg_{\eta^i}} \vecweight^{\vecg_{\eta^i}}$
        \If{$\code{Uniform}(0,1) < \frac{W'}{W}$}
          \State $\vecg^* \gets \vecg$
          \State \textbf{break}
        \Else
          \State $W \gets W - W'$
        \EndIf
      \EndIf
    \EndFor
    \State Obtain the cell partition $\{C_{\eta^i}\}_{i\in[N_c]}$ from $\left(\eta_i\right)_{i\in[n]\setminus\{t\}}$
    \For{$i\in[N_c]$}
      \State Fetch the cardinality vector $\vecg^*_{\eta^i}$ of $\eta^i$ from $\vecg^*$
      \State Randomly partition the cell $C_{\eta^i}$ into $\left\{G_{\eta^i, \pi^j}\right\}_{j\in[N_b]}$ according to $\vecg^*_{\eta^i}$
      \For{$j\in[N_b]$}
        \State $\mu \gets \mu \cup \left\{\pi^j(e_t, e)\right\}_{e\in G_{\eta^i, \pi^j}}$
        \State $\forall e_s\in G_{\eta^i, \pi^j}, \eta_s' \gets \relaxcelltype{\eta_s}{\pi^j}$
      \EndFor
    \EndFor
    \State $\mu\gets \mu \cup \code{DRSampler}(\sentence_T, \domain', \weight, \negweight, \left(\eta_i'\right)_{i\in[n-1]})$
    \State \Return $\mu$
  \end{algorithmic}
\end{algorithm}

\begin{lemma}
  \label{le:drsamplercomplexity}
  The complexity of $\code{DRSampler}(\cdot, \cdot,\cdot, \cdot)$ in Algorithm~\ref{alg:drsampler} is polynomial in the size of the input domain.
\end{lemma}
\begin{proof}
  The algorithm $\code{DRSampler}$ is called at most $n$ times, where $n$ is the size of the domain.
  The main computation of each recursive call is for the loop, where we need to iterate over all $\prod_{i\in[N_c]} |\mathcal{T}{n_{\eta^i}, N_b}|$ possible configurations.
  The size of a configuration space $\mathcal{T}_{M, m}$ is polynomial in $M$, and thus the complexity of this loop is also polynomial in the domain size.
  The other complexity of computing $\mathcal{W}_{\vecn'}$ has been shown to be polynomial in the summation over the vector $\vecn'$, which is clearly smaller than the domain size.
\end{proof}

% Using the definition of configuration space, for every cell $C_\eta$, there are totally $|\mathcal{T}_{|C_\eta|, N_b}|$ possible $\vecg_{\eta}^{\structure_t}$.
% The number of the possible overall configurations over all cells is $\prod_{i\in[N_c]}|\mathcal{T}_{|C_{\eta^i}|, N_b}|$.

\subsubsection{A Lifted WMS for \fotwo{}}

We present our WMS for \fotwo{} in Algorithm~\ref{alg:fo2_wms}.
Given a \fotwo{} sentence $\sentence$ in SNF, the algorithm first obtains the sentence $\sentence_T$ of the general form~\eqref{eq:general_sentence}.
Then the algorithms $\code{OneTypeSampler}$ and $\code{DRSampler}$ then applied successively to sample a skeleton structure of $\generalsentence$.
It is easy to verify that the skeleton structure is also a model of $\sentence$, as it can be regarded as the output of the mapping function in the sound reduction from the WFOMS problem on $\sentence$ to $\generalsentence$.
Since both $\code{OneTypeSampler}$ and $\code{DRSampler}$ have been proved to be polynomial-time in the domain size by Lemma~\ref{lemma:complexity_cell_type} and~\ref{le:drsamplercomplexity}, the WMS in Algorithm~\ref{alg:fo2_wms} is clearly lifted.

\begin{algorithm}[!htb] 
  \caption{$\code{WMS}(\sentence, \domain, \weight, \negweight)$} 
  \label{alg:fo2_wms}
  % \SetKwProg{generate}{Function \emph{generate}}{}{end}
  \begin{flushleft}
    \textbf{INPUT:} An \fotwo{} sentence $\sentence$ of the form \eqref{eq:scott_form}, a domain $\domain=\{e_i\}_{i\in[n]}$ of size $n$, a weighting $(\weight, \negweight)$\\
    \textbf{OUTPUT:} A model $\mu$ of $\sentence$ over $\domain$\\
  \end{flushleft}
  \begin{algorithmic}[1]
    \State Construct $\sentence_T$ from $\sentence$ by~\eqref{eq:tseitin_reduction}
    \State $\forall i\in[n], \beta_i(x)\gets\bigwedge_{k\in[m]}Z_k(x)$
    \State $\generalsentence\gets\sentence_T\land\bigwedge_{i\in[n]}\beta_i(e_i)$
    \State $\left(\tau_i\right)_{i\in[n]}\gets\code{OneTypeSampler}(\generalsentence, \domain, \weight, \negweight)$
    \State $\forall i\in[n], \eta_i \gets (\beta_i, \tau_i)$
    \State $\mu\gets \code{DRSampler}(\sentence_T,\domain,\weight, \negweight,(\eta_i)_{i\in[n]})$
    \State \Return $\mu$
  \end{algorithmic}
\end{algorithm}

\begin{theorem}
  \label{th:main_result}
  The fragment \fotwo{} is domain-liftable under sampling.
\end{theorem}
\begin{proof}[Proof of Theorem~\ref{th:main_result}]
  The proof is directly following from the above and from Lemma~\ref{le:snf_sound}.
\end{proof}

\begin{remark}
  We note that there are several optimizations to our WMS, e.g., heuristically selecting the domain element in $\code{DRSampler}$ so that the algorithm can quickly reach the terminal condition.
  However, the current algorithm is clear and efficient enough to prove our main result, so that we leave the discussion on some of the optimizations to Appendix~\ref{sub:optwms}.
\end{remark}

% \begin{algorithm}[!htb] 
%   \caption{$code{ExSat}(\left(\vecg_{\eta^i}\right)_{i\in N_c}, \beta)$} 
%   \label{alg:ext_sat}
%   % \SetKwProg{generate}{Function \emph{generate}}{}{end}
%   % \begin{flushleft}
%   %   \textbf{INPUT:} A sentence $\generalsentence$ of the form \eqref{eq:general_sentence}, a domain $\domain$ and a weighting $(\weight, \negweight)$\\
%   %   \textbf{OUTPUT:} A model $\mu$ of $\generalsentence$\\
%   % \end{flushleft}
%   \State $g_{\eta^i, \pi}\gets$ the cardinality of $\pi$ in $\vecg$
%   \begin{algorithmic}[1]
%   \end{algorithmic}
% \end{algorithm}

\section{A Genralization to \fotwo{} with Cardinality Constraints}

In this section, we extend our results to \fotwo{} with \textit{cardinality constraints}. 
A single cardinality constraint is a statement of the form $|P|\bowtie q$, where $\bowtie$ is a comparison operator (e.g., $=$, $\le$, $\ge$, $<$, $>$) and $q$ is a natural number. 
These constraints are imposed on the number of distinct positive ground literals in a structure $\structure$ formed by the predicate $P$. 
For example, a structure $\structure$ satisfies the constraint $|P| \le q$ if there are at most $q$ literals for $P$ that are true in $\structure$.
For illustration, we allow cardinality constraints as atomic formulas in the FO formulas, e.g., $(|E| = 2)\land (\forall x\forall y: E(x,y) \Rightarrow E(y,x))$ (its models can be interpreted as undirected graphs with exactly one edge) and the satisfaction relation $\models$ is extended naturally.

The cardinality constraints are not necessarily expressible in FO logic without grounding out the constraint over the domain, and have a strong connection to the fragment of \ctwo{}, which will be introduced in the next section.
In \cite{wangDomainLiftedSamplingUniversal2022}, the authors also extended their WMS (which was originally developed for \ufotwo{}) to handle cardinality constraints. 
However, their method was relatively straightforward whereas the extension to \fotwo{} is more complicated.

Let $\sentence$ be an \fotwo{} sentence and 
\begin{equation}
  \Upsilon := \varphi(|P_1|\bowtie q_1,\dots,|P_M|\bowtie q_M),
  \label{eq:cc}
\end{equation}
where $\varphi$ is a Boolean formula, $\{P_i\}_{i\in[M]}\subseteq \mathcal{P}_\sentence$, and $\forall i\in[M], q_i\in \nat$.
Consider the WFOMS problem on $\sentence\land\Upsilon$ over the domain $\domain$ under $(\weight, \negweight)$.
The overall structure of the sampling algorithm for $\sentence\land\Upsilon$ remains unchanged from Algorithm~\ref{alg:fo2_wms}.
The algorithm still begins with obtaining the general sentence $\generalsentence$ from $\sentence$.
Then the formula $\generalsentence\land\Upsilon$ is fed into $\code{OneTypeSampler}$ and $\code{DRSampler}$ successively to sample the $\mathcal{P}_\sentence$-structure.
Please refer to Appendix~\ref{sub:wmscc} for the detailed algorithm. 
We only describe its modifications to the original one below.

The algorithm of $\code{OneTypeSampler}$ for $\generalsentence\land\Upsilon$ is similar to Algorithm~\ref{alg:sampling_cell_types}, with the main difference being that the computation of \wfomc{} problems, specifically $\symwfomc(\generalsentence, \domain,\weight,\negweight)$ and $\mathcal{W}_\vecn$, now include cardinality constraints $\Upsilon$ in their input sentences.
To account for this change, we slightly modify the definition of $\mathcal{W}_\vecn$ in \eqref{eq:weight_cell_configuration} by taking $\sentence_T\land\bigwedge_{i\in[\widetilde{n}]}\widetilde{\eta}_i\land\Upsilon$ as input, and denote the new term by $\mathcal{W}_{\vecn,\Upsilon}$.
According to Proposition~5 in \cite{kuzelkaWeightedFirstorderModel2021}, the addition of cardinality constraints to a liftable sentence does not affect the liftability of the resulting formula (in terms of \wfomc{} problems). Therefore, the computation of $\mathcal{W}_{\vecn,\Upsilon}$ remains polynomial-time in the domain size, as the original sentence in $\mathcal{W}_\vecn$ was already proven to be liftable.

For the sampling problem conditional on the sampled cell types $\eta_i$, the domain recursive property still holds as we will show in turn.
Given a set $L$ of ground literals and a predicate $P$, let $N(P, L)$ denote the number of positive ground literals for $P$ in $L$.
Given a valid substructure $\structure_t$ of the element $e_t$, denote the 1-type of $e_t$ by $\tau_t$ as usual, let $q_i' = q_i - N(P_i, \structure_t) - N(P_i, \tau_t(e_t))$ for every $i\in[M]$, and define
\begin{equation}
  \Upsilon' = \varphi(|P_1| \bowtie q_1',\dots,|P_M|\bowtie q_M').
  \label{eq:reduced_cc}
\end{equation}
Let $\recursivesentence_{C} = \recursivesentence \land \Upsilon$ and $\recursivesentence_C' = \recursivesentence' \land \Upsilon'$, where $\recursivesentence$ and $\recursivesentence'$ are defined as \eqref{eq:recursive_sentence} and \eqref{eq:reduced_sentence} respectively.
% , and $\domain' = \domain \setminus \{e_t\}$.
Then the reduction from the WFOMS problem on $\recursivesentence$ to $\recursivesentence'$ is sound.
\begin{lemma}
  \label{lemma:ccmodular}
  If $\structure_t$ is valid w.r.t. the WFOMS of $(\bigwedge_{i\in[n]}\eta_i(e_i)\land\sentence_T\land \Upsilon,\domain,\weight, \negweight)$, i.e., $\recursivesentence_C$ is satisfiable, the reduction from the WFOMS of $(\recursivesentence_C, \domain, \weight, \negweight)$ to $(\recursivesentence_C', \domain', \weight, \negweight')$ is sound.
\end{lemma}
\begin{proof}[Proof of Lemma~\ref{lemma:ccmodular}]
  The proof follows the same argument for Lemma~\ref{lemma:modular}.
  The only statement that needs to be argued again is the bijection of the mapping function $f(\structure') = \structure' \cup \structure_t \cup \tau_t(e_t)$.
  Let $\mathfrak{S}^C_1$ and $\mathfrak{S}^C_2$ be the WFOMS of $(\recursivesentence_C, \domain, \weight, \negweight)$ and $(\recursivesentence'_C, \domain', \weight, \negweight)$.
  For any valid $\structure$ of $\mathfrak{S}^C_1$, $\structure$ must satisfy both $\recursivesentence$ and $\Upsilon$.
  It follows that $\structure' = \structure\setminus\structure_t\setminus\{\tau_t(e_t)\}$ satisfies $\recursivesentence'$ and $\Upsilon'$, meaning that $\structure'$ is also valid w.r.t. $\mathfrak{S}_2^C$.
  This establishes the bijection of $f$.
  The remainder of the proof, including the consistency of sampling probability, proceeds exactly the same as Lemma~\ref{lemma:modular}, specifically follows \eqref{eq:modular_wfomc} and \eqref{eq:consistent_weight}.
\end{proof}
The core structure of $\code{DRSampler}$ remains the same as the sound reduction still holds.
However, the recursive call is now made with the reduced sentence $\sentence_T\land\Upsilon'$.
The other slight modifications include:
\begin{itemize}
  \item $\mathcal{W}_{\vecn^{\structure_t}}$ in \eqref{eq:structure_weight} has been replaced with $\mathcal{W}_{\vecn^{\structure_t}, \Upsilon'}$, and 
  \item the validity check for the sampled 2-table configuration in $\code{ExSat}$ now includes an additional check for the well-definedness of the reduced cardinality constraints $\Upsilon'$, returning $\code{False}$ if any $q_i' \notin \nat$ for $i\in[M]$.
\end{itemize}

As discussed above, the extension of our sampling algorithm to handle cardinality constraints in \fotwo{} only slightly increases the complexity of the procedure. 
Furthermore, the computation of $\mathcal{W}_{\vecn,\Upsilon}$ remains polynomial-time in the domain size, meaning that the generalized algorithm is still lifted, and thus proving the liftability under sampling of \fotwo{} with cardinality constraints~\footnote{It is worth noting that the computational complexity of $\mathcal{W}_{\vecn,\Upsilon}$ is independent of the values $q_1,\dots,q_M$ in the cardinality constraints $\Upsilon$, so that the reduction on these constraints does not affect the liftability of the reduced $\mathcal{W}_{\vecn^{\structure_t},\Upsilon'}$.}.
\begin{theorem}
  Let $\sentence$ be an \fotwo{} sentence and $\Upsilon$ of the form \eqref{eq:cc}.
  Then $\sentence\land\Upsilon$ is domain-liftable under sampling.
  \label{th:ccliftability}
\end{theorem}
\begin{proof}
  The proof follows from the discussion above.
\end{proof}

% \lnote{I was thinking of whether to make it a separate section..}

\section{A Further Generalization to \sctwo{}}

With the lifted WMS for \fotwo{} with cardinality constraints, we can further extend our result to the case involving the counting quantifiers $\exists_{=k}$~\cite{gradelTwovariableLogicCounting1997}.
Here, we study the sentences of the form 
\begin{align*}
  \sentence&\land (\forall x\exists_{=k_1}y: \phi_1(x,y))\land\dots\land(\forall x\exists_{=k_{M'}}:\phi_{M'}(x,y))\\
  &\land (\exists_{=k_1'}x\forall y: \phi'_1(x,y))\land\dots\land(\exists_{=k_{M''}'}x\forall y:\phi'_{M''}(x,y)),
\end{align*}
where $\sentence$ is an \fotwo{} sentence and $\exists_{=k}$ is the counting quantifier that specifies the exact number of elements in the domain that satisfy a given formula.
For instance, a structure $\structure$ over a domain $\domain$ satisfies the sentence $\exists_{=k}x: \fotwoformula(x)$, if there are exactly $k$ distinct elements $t_1,\dots,t_k\in\domain$ such that $\structure\models\fotwoformula(t_i)$ for all $i\in[k]$.
We call this fragment \textit{two-variable logic with counting in SNF} \sctwo{}, as its extended conjunction to \fotwo{} sentences resembles SNF.
The presence of counting quantifiers significantly enhances the expressiveness of \sctwo{}, e.g., $k$-regular graphs can be encoded in \sctwo{}, as demonstrated in the introduction.
% \begin{example}
%   \label{ex:k_regular_graph}
%   The model of the sentence 
%   \begin{align*}
%     \sentence = &(\forall x\forall y: (E(x,y)\Rightarrow E(y,x))\land \neg E(x,x))\\
%     &\land (\forall x\exists_{=k} E(x,y))
%   \end{align*}
%   can be viewed as a $k$-regular graph.
%   A WMS of $\sentence$ over a domain of size $n$ under the weighting $\weight(E) = \negweight(E) = 1$ uniformly samples k-regular graphs with $n$ vertices.
%   We note that this problem is extensively studied in the combinatorical community and highly non-trivial, e.g., \cite{cooperSamplingRegularGraphs2007,gaoUniformGenerationRandom2015}.
% \end{example}

Recently, \citet{kuzelkaWeightedFirstorderModel2021} showed that the liftability of \fotwo{} can be generalized to \textit{the fragment of two-variable logic with counting} \ctwo{}, a superset of \sctwo{}, by reducing the \wfomc{} problem on \ctwo{} sentences to \fotwo{} sentences with cardinality constraints.
We demonstrate that this reduction can be also applied to the sampling problem and it is sound, when the fragment is restricted to be $\sctwo{}$.
\begin{lemma}
  For any WFOMS $\mathfrak{S} = (\Phi, \domain,\weight, \negweight)$ where $\Phi$ is a \sctwo{} sentence, there exists a WFOMS $\mathfrak{S}'=(\sentence'\land\Upsilon,\domain,\weight', \negweight')$, where $\sentence'$ is an \fotwo{} sentence, $\Upsilon$ denotes cardinality constraints of the form \eqref{eq:cc} and both $\sentence'$ and $\Upsilon$ are independent of $\domain$, such that the reduction from $\mathfrak{S}$ to $\mathfrak{S}'$ is sound.
  \label{le:ctworeduction}
\end{lemma}
The proof follows a similar technique used in \cite{kuzelkaWeightedFirstorderModel2021}, and the details are deferred to Appendix~\ref{sub:sctworeduction}.
We note here that further generalizing this result to the general \ctwo{} sentences is infeasible, since the original reduction used in \cite{kuzelkaWeightedFirstorderModel2021} for \ctwo{} sentences introduced some negative weights on predicates.
However, it can be established that the domain-liftability under sampling of \ctwo{} can be demonstrated by directly applying our domain recursion sampling method without resorting to the reduction to cardinality constraints. 
For a more detailed discussion, please refer to Appendix~\ref{sub:ctwoliftable}.

Since \fotwo{} with cardinality constraints has been proved to be liftable under sampling, it is easy to prove the liftability under sampling of the \sctwo{} fragment.
\begin{theorem}
  The fragment of \sctwo{} is domain-liftable under sampling.
  \label{th:ctwoliftability}
\end{theorem}
\begin{proof}
  The proof follows from Lemma~\ref{le:ctworeduction} and Theorem~\ref{th:ccliftability}.
\end{proof}
Moreover, one can further introduce additional cardinality constraints into \sctwo{} without degrading its liftability under sampling.
\begin{corollary}
  Let $\Phi$ be a \sctwo{} sentence and $\Upsilon$ of the form \eqref{eq:cc}. Then $\Phi\land\Upsilon$ is domain-liftable under sampling.
  \label{co:sctwoccliftability}
\end{corollary}

\section{Experimental Results}
\label{sec:exp}

We conducted several experiments to evaluate the performance and correctness of our sampling algorithms.
% In this section, we present some of the preliminary results.
All algorithms were implemented in Python and  the experiments were performed on a computer with an 8-core Intel i7 3.60GHz processor and 32 GB of RAM~\footnote{The code can be found in \url{https://github.com/lucienwang1009/lifted_sampling_fo2}}.

Many sampling problems can be expressed as WFOMS problems.
Here we consider two typical ones.
\begin{itemize}
  \item \textbf{Sampling combinatorial structures}: the uniform generation of some combinatorial structures can be directly reduced to a WFOMS, e.g., the uniform generation of \textit{graphs with no isolated vertices} and \textit{$k$-regular graphs} in Examples~\ref{ex:non-isolated_graph} and the introduction.
  We added four more combinatorial sampling problems to these two for evaluation: \textit{functions}, \textit{functions w/o fix-points} (i.e., the functions $f$ satisfying $f(x) \neq x$), \textit{permutations} and \textit{permutations w/o fix-points}.
  The details of these problems are described in Appendix~\ref{sub:expsettings}.
  \item \textbf{Sampling from MLNs}: our algorithms can be also applied to sample possible worlds from MLNs.
  % An MLN is a finite set of tuples $(\weight, \formula)$, where $\weight$ is either a real number or $\infty$, and $\formula$ is a formula in first-order logic.
  An MLN defines a distribution over structures (i.e., possible worlds in SRL literature), and its respective sampling problem is to randomly generate possible worlds according to this distribution.
  There is a standard reduction from the sampling problem of an MLN to a WFOMS problem (see Append~\ref{sub:expsettings} and also \cite{wangDomainLiftedSamplingUniversal2022}).
  We used two MLNs in our experiments: 1) A variant of the classic \textit{friends-smokers} MLN with the constraint that every person has at least one friend:
    \begin{align*}
      \{(\infty, \neg fr(x,x)), (\infty, fr(x,y)\Rightarrow fr(y,x)), (0, sm(x)),\\
      (0.2, fr(x,y)\land sm(x)\Rightarrow sm(y)), (\infty, \exists y: fr(x,y))\}.
    \end{align*}
    2) The \textit{employment} MLN used in \cite{vandenbroeck2014Proc.FourteenthInt.Conf.Princ.Knowl.Represent.Reason.}:
    \begin{equation*}
      \{(1.3, \exists y: workfor(x,y)\lor boss(x))\},
    \end{equation*}
    which states that with high probability, every person either is employed by a boss or is a boss.
  The details about the reduction from sampling from MLNs to WFOMS and the corresponding WFOMS problems of these two MLNs can be found in Appendix~\ref{sub:expsettings}.
\end{itemize}

\subsection{Correctness}

We first examine the correctness of our implementation on the uniform generation of combinatorial structures over small domains, where exact sampling is feasible via enumeration-based techniques; 
we choose the domain size of $5$ for evaluation.
To serve as a benchmark, we implemented a simple ideal uniform sampler, denoted by IS, by enumerating all the models and then drawing samples uniformly from these models.
For each combinatorial structure encoded into an \fotwo{} sentence $\sentence$, a total of $100 \times |\fomodels{\sentence}{\domain}|$ models were generated from both IS and our WMS.
Figure~\ref{fig:uniformity} depicts the model distribution produced by these two algorithms---the horizontal axis represents models numbered lexicographically, while the vertical axis represents the generated frequencies of models.
The figure suggests that the distribution generated by our WMS is indistinguishable from that of IS.
Furthermore, a statistical test on the distributions produced by WMS was performed, and no statistically significant difference from the uniform distribution was found.
The details of this test can be found in Appendix~\ref{sub:expresults}.
% \lnote{The Jensen-Shannon distance is token from Unigen's paper. I feel like it's more easy to understand than DKW inequality.}
% In particular, the Jensen-Shannon distances between the distributions from IS and WMS are $0.047$, $0.049$, $0.050$, $0.045$ and $0.051$ respectively.

\begin{figure}
  \DoubleColumn
  \centerline{\includegraphics[width=.49\textwidth]{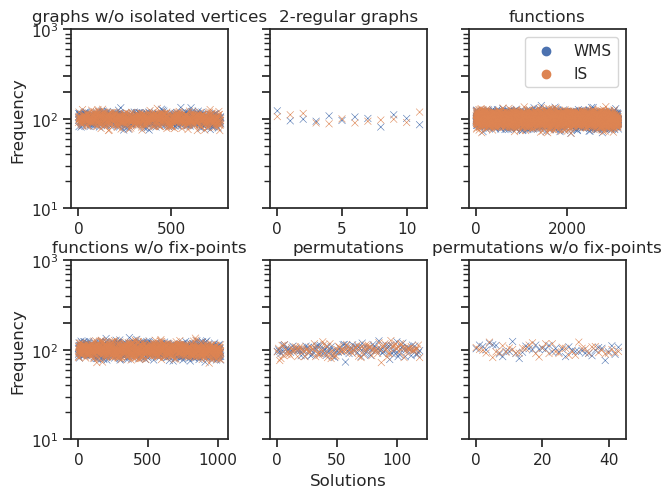}}
  \DoubleColumnEnd
  \SingleColumn
  \centerline{\includegraphics[width=.69\textwidth]{figs/uniformity.png}}
  \SingleColumnEnd
  \caption{Uniformity comparison between an ideal sampler (IS) and our WMS.}
  \label{fig:uniformity}
\end{figure}

\begin{figure}
  \begin{subfigure}[b]{0.49\textwidth}
      \centerline{
        \includegraphics[width=\textwidth]{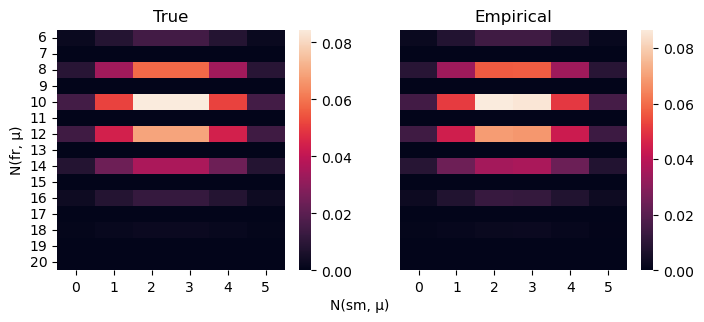}
      }
      \caption{friends-smokers}
      \label{fig:fr-sm}
  \end{subfigure}
  \begin{subfigure}[b]{0.49\textwidth}
      \centerline{
        \includegraphics[width=\textwidth]{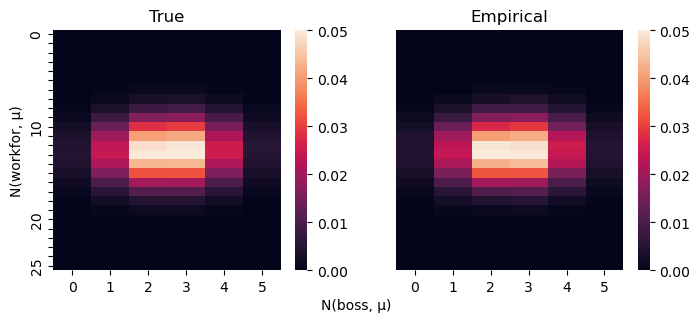}
      }
      \caption{employment}
      \label{fig:employ}
  \end{subfigure}
  \caption{Conformity testing for the count distribution of MLNs.}
  \label{fig:mlnconformity}
\end{figure}

For sampling problems from MLNs, enumerating all the models is infeasible even for a domain of size $5$, e.g., there are $2^{2^5 + 5} = 2^{37}$ models in the employment MLN.
That is why we test the \textit{count distribution} of vocabulary for these two MLNs.
Instead of specifying the probability of each model, the count distribution only tells us how probably a certain number of predicates are interpreted to be true in the models.
An advantage of testing count distributions is that they can be efficiently computed for our MLNs.
Please refer to~\cite{kuzelkaWeightedFirstorderModel2021} for more details about count distributions.
We also note that the conformity of count distribution is a necessary condition for the correctness of algorithms.
We keep the domain size to be $5$ and sampled $10^5$ models from friends-smokers and employment MLNs respectively.
The empirical distributions of count-statistics, along with the true count distributions, are shown in Figure~\ref{fig:mlnconformity}.
% Recall that the notation $N(P,\mu)$ indicates the number of positive ground literals for the predicate $P$ in the model $\mu$.
It is easy to check the conformity of the empirical distribution to the true one from the figure.
% We also computed the Jensen-Shannon distances between the empirical and true distributions for these two MLNs, resulting in $0.008$ and $0.013$ respectively.
The statistical test was also performed on the count distribution, and the results confirm the conclusion drawn from the figure (also see Appendix~\ref{sub:expresults}).

\subsection{Performance}

To evaluate the performance, we compared our weighted model samplers with Unigen~\cite{chakrabortyScalableNearlyUniform2013,soosTintedDetachedLazy2020}, the state-of-the-art approximate sampler for Boolean formulas.
A WFOMS problem can be reduced to a sampling problem of the Boolean formula by grounding the input sentence over the given domain.
Since Unigen only works for uniform sampling, we employed the technique in~\cite{chakrabortyWeightedUnweightedModel2015} to encode the weighting function in the WFOMS problem into a Boolean formula.

\begin{figure}[!tb]
   \centering
   \includegraphics[width=.95\textwidth]{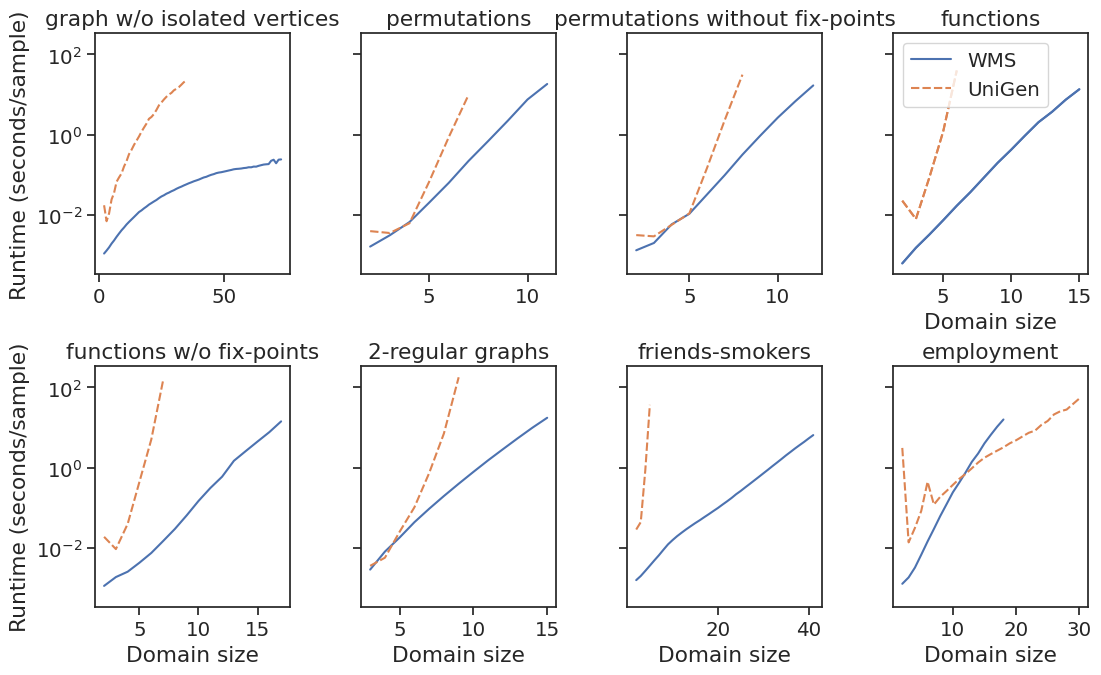}
   \caption{Performance of WMS versus UniGen.}
   \label{fig:performance}
\end{figure}

For each sampling problem, we randomly generated $1000$ models by our WMS and Unigen respectively and computed the average sampling time of one model.
The performance comparison is shown in Figure~\ref{fig:performance}.
In most cases, our approach is much faster than UniGen.
The exception in the employment MLN, where UniGen performed better than WMS, is likely due to the simplicity of this specific instance for its underlying SAT solver.
This coincides with the theoretical result that our WMS is polynomial-time in the domain size, while UniGen usually needs amounts of expensive SAT calls on the grounding formulas.

\section{Conclusion and Future Work}

In this paper, we prove the domain-liftability under sampling of \fotwo{} by presenting a novel and efficient approach to its symmetric weighted first-order model sampling problems.
The result is further extended to the fragment of \sctwo{} with the presence of counting constraints.
The widespread applicability of WFOMS renders the proposed approach a promising candidate to serve as a universal paradigm for a plethora of sampling problems.

A potential avenue for further research is to expand the methodology presented in this paper to encompass more expressive first-order languages. 
Specifically, the utilization of the domain recursion scheme employed in this study could be extended beyond the confines of the \fotwo{} and \sctwo{}, as its analogous counterpart in WFOMC has been demonstrated to be effective in proving the domain-liftability of the fragments $\mathbf{S}^2\fotwo{}$ and $\mathbf{S}^2\mathbf{RU}$~\cite{kazemiDomainRecursionLifted2017}.

In addition to extending the input logic, other potential directions for future research include incorporating elementary axioms, such as tree axiom~\cite{vanbremenLiftedInferenceTree2021} and linear order axiom~\cite{toth2022ArXivPrepr.ArXiv221101164}, as well as more general weighting functions that involve negative weights. 
However, it is important to note that these extensions would likely require a more advanced and nuanced approach than the one proposed in this paper, and may present significant challenges.

Finally, the lower complexity bound of WFOMS is also an interesting open problem.
A direct implication from the infeasibility of WFOMC in ~\cite{jaegerLowerComplexityBounds2015} suggests that there is unlikely for an (even approximate) lifted WMS to exist for full first-order logic. 
However, the establishment of a tighter lower bound for fragments of FO, such as $\mathbf{FO}^3$, remains an unexplored and challenging area that merits further investigation.

\section*{Acknowledgement}

The authors would like to thank the anonymous reviewers for their helpful comments. Yuanhong Wang and Juhua Pu are supported by the National Key R\&D Program of China (2021YFB2104800) and the National Science Foundation of China (62177002). Ond\v{r}ej Ku\v{z}elka's work is supported by the Czech Science Foundation project 20-19104Y and partially also 23-07299S (most of the work was done before the start of the latter project)

\clearpage

\appendix

\section{Appendix}

\subsection{Scott Normal Forms}
\label{sub:snf}

We briefly describe the transformation of \fotwo{} formulas to SNF and prove the soundness of its corresponding reduction on the WFOMS problems.
The process is well-known, so we only sketch the related details.

Let $\sentence$ be a sentence of \fotwo{}.
To put it into SNF, consider a subformula $\varphi(x) = Qy: \phi(x, y)$, where $Q\in\{\forall, \exists\}$ and $\phi$ is quantifier-free.
Let $A_\varphi$ be a fresh unary predicate\footnote{If $\varphi(x)$ has no free variables, e.g., $\exists x: \phi(x)$, the predicate $A_\varphi$ is nullary.} and consider the sentence
\begin{equation*}
  \forall x: (A_\varphi(x) \Leftrightarrow (Qy: \phi(x,y)))
\end{equation*}
which states that $\varphi(x)$ is equivalent to $A_\varphi(x)$.
Let $Q'$ denote the dual of $Q$, i.e., $Q' = \{\forall,\exists\}\setminus \{Q\}$, this sentence can be seen equivalent to
\begin{align*}
  \sentence' := &\forall x Qy: (A_\varphi(x) \Rightarrow \phi(x, y)) \\
  &\land \forall x Q'y: (\phi(x,y) \Rightarrow A_\varphi(x)).
\end{align*}

Let 
\begin{equation*}
  \sentence'' = \sentence' \land \sentence[\varphi(x)/A_\varphi(x)],
\end{equation*} where $\sentence[\varphi(x)/A_\varphi(x)]$ is obtained from $\sentence$ by replacing $\varphi(x)$ with $A_\varphi(x)$.
For any domain $\domain$, every model of $\sentence''$ over $\domain$ can be mapped to a unique model of $\sentence$ over $\domain$.
The bijective mapping function is simply the projection $\proj{\cdot}{\mathcal{P}_\sentence}$.
Let both the positive and negative weights of $A_\varphi$ be $1$ and denote the new weighting functions as $\weight'$ and $\negweight'$.
It is clear that the reduction from $(\sentence, \domain, \weight, \negweight)$ to $(\sentence'', \domain, \weight', \negweight')$ is sound.
Repeat this process from the atomic level and work upwards until the sentence is in SNF.
The whole reduction remains sound due to the transitivity of soundness.

\subsection{A Sound Reduction from \sctwo{} to \fotwo{} with Cardinality Constraints}
\label{sub:sctworeduction}

In this section, we show the sound reduction from a WFOMS problem on \sctwo{} sentence to a WFOMS problem on \fotwo{} sentence with cardinality constraints.

We first need the following two lemmas.
\begin{lemma}
  Let $\sentence$ be a first-order logic sentence, and let $\domain$ be a domain.
  Let $\Phi$ be a first-order sentence with cardinality constraints, defined as follows:
  \begin{align*}
    \Pi := &(|P|=k\cdot |\domain|)\\
    &\land (\forall x\forall y: P(x,y)\Leftrightarrow (R_1^P(x,y)\lor \dots\lor R_k^P(x,y)))\\
    &\land \bigwedge_{i\in[k]}(\forall x\exists y: R_i^P(x,y))\\
    &\land \bigwedge_{i,j\in[k]: i\neq j} (\forall x\forall y: \neg R_i^P(x,y)\lor \neg R_j^P(x,y)),
  \end{align*}
  where $R_i^P$ are auxiliary predicates not in $\mathcal{P}_\sentence$ with weight $\weight(R_i^P) = \negweight(R_i^P) = 1$.
  Then the reduction from the WFOMS $(\sentence\land \forall x\exists_{=k} y: P(x,y), \domain, \weight, \negweight)$ to $(\sentence\land\Pi, \domain, \weight, \negweight)$ is sound.
  \label{le:forallexistsred}
\end{lemma}
\begin{proof}
  Let $f(\cdot) = \proj{\cdot}{\mathcal{P}_\sentence\cup\{P\}}$ be a mapping function.
  We first show that $f$ is from $\fomodels{\sentence\land\Pi}{\domain}$ to $\fomodels{\sentence\land\forall x\exists_{=k}: P(x,y)}{\domain}$: if $\structure\models \sentence\land \Pi$ then $f(\structure)\models\sentence\land\forall x\exists_{=k}y: P(x,y)$.
  
  The sentence $\Pi$ means that for every $c_1, c_2\in\domain$ such that $P(c_1,c_2)$ is true, there is exactly one $i\in[k]$ such that $R_i^P(c_1, c_2)$ is true.
  Thus we have that $\sum_{i\in[k]} |R_i^P| = |P| = k\cdot|\domain|$, which together with $\bigwedge_{i\in[k]}\forall x\exists y: R_i^P(x,y)$ implies that $|R_i^P| = k$ for $i\in[k]$.
  We argue that each $R_i^P$ is a function predicate in the sense that $\forall x\exists_{=1}y: R_i^P(x,y)$ holds in any model of $\sentence\land\Pi$.
  Let us suppose, for contradiction, that $(\forall x\exists y: R_i^P(x,y))\land (|R_i^P| = k)$ holds but there is some $a\in\domain$ such that $R_i^P(a, b)$ and $R_i^P(a,b')$ are true for some $b \neq b'\in\domain$.
  We have $|\{(x,y)\in\domain^2\mid R_i^P(x,y)\land x\neq a\}| \ge |\domain| - 1$ by the fact $\forall x\exists y: R_i^P(x,y)$. 
  It follows that $|R_i^P| \ge |\{(x,y)\in\domain^2\mid R_i^P(x,y)\land x\neq a\}| + 2 > |\domain|$, which leads to a contradiction.
  Since all of $R_i^P$ are function predicates, it is easy to check $\forall x\exists_{=k}y: P(x,y)$ must be true in any model $\mu$ of $\sentence\land\Pi$, i.e., $f(\mu)\models \sentence\land\forall x\exists_{=k}y: P(x,y)$.

  To finish the proof, one can easily show that, for every model $\mu\in\fomodels{\sentence\land\forall x\exists_{=k}y: P(x,y)}{\domain}$, there are exactly $(k!)^{|\domain|}$ models $\mu'\in\fomodels{\sentence\land\Pi}{\domain}$ such that $f(\mu') = \mu$.
  The reason for this is that 1) if, for any $a\in\domain$, we permute $b_1, b_2, \dots, b_k$ in $R_1^P(a, b_1), R_2^P(a, b_2), \dots, R_k^P(a, b_k)$ in the model $\mu'$, we get another model of $\sentence\land\Pi$, and 2) up to these permutations, the predicates $R_i^P$ in $\mu'$ are determined uniquely by $\mu$.
  Finally, the weights of all these $\mu'$s are the same as those of $\mu$, and we can write
  \DoubleColumn
  \begin{align*}
    &\sum_{\substack{\mu'\in\fomodels{\sentence\land\Pi}{\domain}:\\f(\mu')=\mu}} \pro[\mu'\mid\sentence\land\Pi]\\
    &\qquad =\frac{\sum_{\substack{\mu'\in\fomodels{\sentence\land\Pi}{\domain}:\\f(\mu')=\mu}} \mu'}{\symwfomc(\sentence\land\Pi,\domain,\weight, \negweight)}\\
    &\qquad =\frac{(k!)^{|\domain|}\cdot \typeweight{\mu}}{(k!)^{|\domain|}\cdot \symwfomc(\sentence\land\forall x\exists_{=k}y: P(x,y), \domain, \weight, \negweight)}\\
    &\qquad =\frac{\typeweight{\mu}}{\symwfomc(\sentence\land\forall x\exists_{=k}y: P(x,y), \domain, \weight, \negweight)}\\
    &\qquad =\pro[\mu\mid\sentence\land\forall x\exists_{=k}y: P(x,y)],
  \end{align*}
  \DoubleColumnEnd
  \SingleColumn
  \begin{align*}
    \sum_{\substack{\mu'\in\fomodels{\sentence\land\Pi}{\domain}:\\f(\mu')=\mu}} \pro[\mu'\mid\sentence\land\Pi] &= \frac{\sum_{\substack{\mu'\in\fomodels{\sentence\land\Pi}{\domain}:\\f(\mu')=\mu}} \mu'}{\symwfomc(\sentence\land\Pi,\domain,\weight, \negweight)}\\
    & =\frac{(k!)^{|\domain|}\cdot \typeweight{\mu}}{(k!)^{|\domain|}\cdot \symwfomc(\sentence\land\forall x\exists_{=k}y: P(x,y), \domain, \weight, \negweight)}\\
    & =\frac{\typeweight{\mu}}{\symwfomc(\sentence\land\forall x\exists_{=k}y: P(x,y), \domain, \weight, \negweight)}\\
    & =\pro[\mu\mid\sentence\land\forall x\exists_{=k}y: P(x,y)],
  \end{align*}
  \SingleColumnEnd
  which completes the proof.
\end{proof}

\begin{lemma}
  Let $\sentence$ be a first-order logic sentence, $\domain$ be a domain, and $P$ be a predicate.
  Then the WFOMS $(\sentence\land\forall_{=k}\forall y: P(x,y), \domain, \weight, \negweight)$ can be reduced to $(\sentence\land(|U| = k)\land (\forall x: U(x)\Leftrightarrow (\forall y: P(x,y))), \domain, \weight, \negweight)$, where $U$ is an auxiliary unary predicate with weight $\weight(U) = \negweight(U) = 1$, and the reduction is sound.
  \label{le:existsforallred}
\end{lemma}
\begin{proof}
  The proof is straightforward.
\end{proof}

\begin{proof}[Proof of Lemma~\ref{le:ctworeduction}]
  We can first get rid of all formulas of the form $\exists_{=k}x \forall y: P(x,y)$ by repeatedly using Lemma~\ref{le:existsforallred}.
  Then we can use Lemma~\ref{le:forallexistsred} repeatedly to eliminate the formulas of the form $\forall x\exists_{=k} y: P(x,y)$.
  The whole reduction is sound due to the transitivity of soundness.
\end{proof}

\subsection{Applying Domain Recursion Scheme on \ctwo{} is Possible}
\label{sub:ctwoliftable}

The \ctwo{} sentences that we need to handle are of the form
\begin{equation}
  \sentence \land \bigwedge_{k\in[q]}(\forall x: A_k(x) \Leftrightarrow (\exists_{=m_k} y: R_k(x,y))),
  \label{eq:ctwonromal}
\end{equation}
where $\sentence$ is a \fotwo{} sentence, each $R_k(x,y)$ is an atomic formula, and each $A_k$ is an auxiliary Tseitin predicate.
Any WFOMS problem on \ctwo{} can be reduced to a new one, whose input sentence is of the above form and $\max_{k\in[q]} m_k \le |\domain|$, by the following steps:
\begin{itemize}
  \item Convert each counting-quantified formula of the form $\exists_{\ge m} y: \extformula(x,y)$ to $\neg (\exists_{\le m-1} y: \extformula(x,y))$.
  \item Decompose each $\exists_{\le m} y: \extformula(x,y)$ into $(\forall y: \neg \extformula(x,y)) \land \bigvee_{i\in[m]} (\exists_{=i} y: \extformula(x,y))$.
  \item Replace each subformula $\exists_{=m} y: \extformula(x,y)$, where $m > |\domain|$, with $\code{False}$.
  \item Starting from the atomic level and working upwards, replace any subformula $\exists_{=m} \extformula(x,y)$, where $\extformula(x,y)$ is a formula that does not contain any counting quantifier, with $A(x)$; and append $\forall x\forall y: R(x,y) \Leftrightarrow \extformula(x,y)$ and $\forall x: A(x)\Leftrightarrow (\exists_{=m}y: R(x,y))$, where $R$ is an auxiliary binary predicate, to the original sentence.
\end{itemize}
It is easy to check that the reduction presented above is sound and independent of the domain size if the domain size is greater than the maximum counting parameter $m$ in the input sentence.\footnote{This condition does not change the data complexity of the problem, as all the counting parameters in the sentence are considered constants but not the input of the problem.}

We first sample the 1-types of each element from the sentence \eqref{eq:ctwonromal} so that all the predicates $A_k$ will be eliminated.
The resulting WFOMS is then defined on the following sentence:
\begin{equation*}
  \sentence_0 \land \bigwedge_{k\in[q]} \left(\bigwedge_{e\in\domain_k^\exists} \exists_{=m_k} y: R_k(e, y)\land\bigwedge_{e\in\domain_k^\nexists} \neg (\exists_{=k}y: R_k(e,y))\right),
\end{equation*}
where $\sentence_0$ is the simplified sentence of $\sentence$ by replacing all its unary literals with their truth values, $\domain_k^\exists$ contains precisely the elements with positive sampled literals $A_k(e)$, and $\domain_k^\nexists = \domain \setminus \domain_k^\exists$.

We need to consider a more general WFOMS problem to apply our domain recursion scheme.
For each counting quantified formula $\exists_{=k} y: R_k(x,y)$, we introduce $2 m_k$ new unary predicates $Z^\exists_{k, 1}, Z^\exists_{k, 2},\dots, Z^\exists_{k, m_k}, Z^\nexists_{k, 1}, Z^\nexists_{k, 2}, \dots, Z^\nexists_{k, m_k}$, and append the conjunction of
\begin{equation*}
  \forall x: \left(Z^\exists_{k, t}(x) \Leftrightarrow (\exists_{=t} y: R_k(x,y))\right) \land \left(Z^\nexists_{k, t}(x)\Leftrightarrow \neg (\exists_{=t} y: R_k(x,y))\right)
\end{equation*}
over $t\in[m_k]$ to $\sentence_0$, resulting in a new sentence $\sentence_1$.
The more general WFOMS is then defined on
\begin{equation}
  \sentence_1 \land \bigwedge_{i\in[n]} \nu_i(e_i),
  \label{eq:ctwogeneral}
\end{equation}
where each $\nu_i(x)$ is a quantifier-free conjunction over a subset of $\{Z_{k, t}^\exists(x)\}_{t\in[m_k]}\cup \{Z_{k, t}^\nexists(x)\}_{t\in[m_k]}$.
It is easy to check that the original WFOMS of \eqref{eq:ctwonromal} is reducible to the more general WFOMS problem, and the reduction is sound and independent of the domain size.

We show that the domain recursion scheme is still applicable to the WFOMS of \eqref{eq:ctwogeneral}.
We only provide an intuition here while leaving the details for the future version of this paper.
Additionally, we hope to discover a more practical and efficient solution in the future that would introduce fewer unary predicates, despite the current approach being domain-lifted.

The intuition is that we can view $\nu_i(x)$ as the ``block type'' similar to what we have done in the WMS of \fotwo{}.
Then the domain recursion strategy is applied, first sampling the substructure of an element, and then updating each block type accordingly.
The updated block types can still be represented by the unary predicates $Z_{k,t}^\exists$ and $Z_{k,t}^\nexists$, and the new sampling problem is reducible to a new WFOMS of the general form.
Following the similar argument for sampling \fotwo{}, the corresponding WMS for \ctwo{} is also lifted, which means that the full fragment of \ctwo{} is liftable under sampling.

\subsection{Missing Details of Experiments}

\subsubsection{Experiment Settings}
\label{sub:expsettings}

\paragraph{Sampling Combinatorial Structures}

The corresponding WFOMS problems for the uniform generation of combinatorial structures used in our experiments are presented as follows. 
The weighting functions $\weight$ and $\negweight$ map all predicates to $1$.
\begin{itemize}
  \item Functions: 
  \begin{equation*}
    \forall x\exists_{=1} y: f(x,y).
  \end{equation*}
  \item Functions w/o fix points:
  \begin{equation*}
    (\forall x\exists_{=1}y: f(x,y))\land (\forall x: \neg f(x,x)).
  \end{equation*}
  \item Permutations:
  \begin{equation*}
    (\forall x\exists_{=1}y: Per(x,y)) \land (\forall y\exists_{=1}x: Per(x,y)).
  \end{equation*}
  \item Permutation without fix-points:
  \DoubleColumn
  \begin{align*}
    &(\forall x\exists_{=1}y: Per(x,y)) \land (\forall y\exists_{=1}x: Per(x,y))\\
    &\land (\forall x: \neg Per(x,x)).
  \end{align*}
  \DoubleColumnEnd
  \SingleColumn
  \begin{equation*}
    (\forall x\exists_{=1}y: Per(x,y)) \land (\forall y\exists_{=1}x: Per(x,y))\land (\forall x: \neg Per(x,x)).
  \end{equation*}
  \SingleColumnEnd
\end{itemize}

\paragraph{Sampling from MLNs}

An MLN is a finite set of weighted first-order formulas $\{(\weight_i, \formula_i)\}_{i\in[m]}$, where each $\weight_i$ is either a real-valued weight or $\infty$, and $\formula_i$ is a first-order formula.
Let $\mathcal{P}$ be the vocabulary of $\formula_1,\formula_2,\dots,\formula_m$.
An MLN $\Phi$ paired with a domain $\domain$ induces a probability distribution over $\mathcal{P}$-structures (also called possible worlds):
\begin{equation*}
  p_{\mln, \domain}(\world) := \begin{cases}
    \frac{1}{Z_{\mln,\domain}}\exp\left(\sum_{(\formula, \weight)\in\mln_\real}\weight\cdot \#(\formula, \world)\right) & \textrm{if } \world\models \mln_\infty\\
    0 & \textrm{otherwise}
  \end{cases}
\end{equation*}
where $\mln_\real$ and $\mln_\infty$ are the real-valued and $\infty$-valued formulas in $\mln$ respectively, and $\#(\formula, \world)$ is the number of groundings of $\formula$ satisfied in $\world$.
The sampling problem on an MLN $\mln$ over a domain $\domain$ is to randomly generate a possible world $\world$ according to the probability $p_{\mln, \domain}(\world)$.

The reduction from the sampling problems on MLNs to WFOMS can be performed as follows.
For every real-valued formula $(\formula_i, \weight_i)\in\mln_\real$, where the free variables in $\formula_i$ are $\vecx$, we introduce a novel auxiliary predicate $\xi_i$ and create a new formula $\forall \vecx: \xi_i(\vecx)\Leftrightarrow \formula_i(\vecx)$.
For formula $\formula_i$ with infinity weight, we instead create a new formula $\forall \vecx: \formula_i(\vecx)$.
Denote the conjunction of the resulting set of sentences by $\sentence$, and set the weighting function to be $\weight(\xi_i)=\exp(\weight_i)$ and $\negweight(\xi_i) = 1$, and for all other predicates, we set both $\weight$ and $\negweight$ to be $1$.
Then the sampling problem on $\mln$ over $\domain$ is reduced to the WFOMS $(\sentence, \domain, \weight, \negweight)$.

By the reduction above, we can write the two MLNs used in our experiments to WFOMS problems.
The weights of predicates are all set to be $1$ unless otherwise specified.
\begin{itemize}
  \item Friends-smokers MLN: the reduced sentence is
  \begin{align*}
    &(\forall x: \neg fr(x,x)\land sm(x))\\
    &\land (\forall x\forall y: fr(x,y)\Leftrightarrow fr(y,x))\\
    &\land (\forall x\forall y: \xi(x,y)\Leftrightarrow (fr(x,y)\land sm(x)\Rightarrow sm(y)))\\
    &\land (\forall x\exists y: fr(x,y)),
  \end{align*}
  and the weight of $\xi$ is set to be $\weight(\xi) = \exp(0.2)$.
  \item Employment MLN: the corresponding sentence is
  \begin{align*}
    \forall x: \xi(x)\Leftrightarrow (\exists y: workfor(x,y)\lor boss(x)),
  \end{align*}
  and the weight of $\xi$ is set to be $\exp(1.3)$.
\end{itemize}

\subsubsection{More Experimental Results}
\label{sub:expresults}

\paragraph{The Kolmogorov–Smirnov Test}
% \label{subsub:kstest}

We utilized the Kolmogorov-Smirnov (KS) test~\cite{massey1951kolmogorov} to validate the conformity of the (count) distributions produced by our algorithm to the reference distributions.
The KS test used here is based on the multivariate Dvoretzky–Kiefer–Wolfowitz (DKW) inequality recently proved by \cite{naaman2021tight}.

Let $\mathbf{X}_1 = (X_{1i})_{i\in[k]}, \mathbf{X}_2 = (X_{2i})_{i\in[k]},\dots,\mathbf{X}_n = (X_{ni})_{i\in[k]}$ be $n$ real-valued independent and identical distributed multivariate random variables with cumulative distribution function (CDF) $F(\cdot)$.
Let $F_n(\cdot)$ be the associated empirical distribution function defined by
\begin{equation*}
  F_n(\vecx) := \frac{1}{n}\sum_{i\in[n]} \indicator_{X_{i1}\le x_1,X_{i2}\le x_2,\dots, X_{ik}\le x_k}, \qquad \vecx\in\real^k.
\end{equation*}
The DKW inequality states
\begin{equation}
  \pro\left[\sup_{\vecx\in\real^k}|F_n(\vecx) - F(\vecx)| > \epsilon\right] \ge (n+1)ke^{-2n\epsilon^2}
  \label{eq:dkw}
\end{equation}
for every $\epsilon, n, k > 0$.
When the random variables are univariate, i.e., $k=1$, we can replace $(n+1)k$ in the above probability bound by a tighter constant $2$.

\begin{table}[!htb]
  \caption{The Kolmogorov-Smirnov Test}
  \label{ta:kstest}
  \centering
  \begin{tabular}{|c|c|c|}
    \hline
    Problem                      & Maximum deviation & Upper bound \\ \hline \hline
    graphs w/o isolated vertices & 0.0036            & 0.0049      \\ \hline
    2-regular graphs             & 0.0065            & 0.0069      \\ \hline
    functions                    & 0.0013            & 0.0024       \\ \hline
    functions w/o fix-points     & 0.0027            & 0.0042      \\ \hline
    permutations                 & 0.0071            & 0.0124      \\ \hline
    permutations w/o fix-points  & 0.019            & 0.02      \\ \hline
    friends-smokers              & 0.0021            & 0.0087      \\ \hline
    employmenet                  & 0.0030            & 0.0087      \\ \hline
  \end{tabular}
\end{table}

In the KS test, the null hypothesize is that the samples $\mathbf{X}_1, \mathbf{X}_2,\dots,\mathbf{X}_n$ are distributed according to some reference distribution, whose CDF is $F(\cdot)$.
Then by \eqref{eq:dkw}, with probability $1-\alpha$, the maximum deviation $\sup_{\vecx\in\real^k}|F_n(\vecx) - F(\vecx)|$ between empirical and reference distributions is bounded by $\epsilon = \sqrt{\ln(k(n+1)/\alpha)/2n}$ ($\sqrt{\ln(2 / \alpha) / 2n}$ for the univariate case).
If the actual value of the maximum deviation is larger than $\epsilon$, we can reject the null hypothesis at the confidence level $\alpha$.
Otherwise, we cannot reject the null hypothesis, i.e., the empirical distribution of the samples is not statistically different from the reference one.
In our experiments, we choose $\alpha = 0.05$ as a significant level.

For the uniform generation of combinatorial structures, we assigned each model a lexicographical number and treated the model index as a random variable with a discrete uniform distribution.
For the sampling problems of MLNs, we test their count distributions against the true count distributions.
Table~\ref{ta:kstest} shows the maximum deviation between the empirical and reference cumulative distribution functions, along with the upper bound set by the DKW inequality. 
As shown in Table~\ref{ta:kstest}, all maximum deviations are within their respective upper bounds.
Therefore, we cannot reject any null hypotheses, i.e., there is no statistically significant difference between the two sets of distributions.

\subsection{Missing Details of WMS}

\subsubsection{Optimizations for WMS}
\label{sub:optwms}

There exist several optimizations to make it more practical. 
Here, we present some of them that are used in our implementation.
\begin{itemize}
  \item The complexity of $\code{DRSampler}$ heavily depends on the recursion depth.
  In our implementation, when selecting a domain element $e_t$ for sampling its substructure, we always chose the element with the ``strongest'' existential constraint that contains the most Tseitin atoms $Z_k(x)$.
  It would help $\code{DRSampler}$ fast reach the condition that the existential constraint for all elements is $\top$.
  In this case, $\code{DRSampler}$ will invoke the more efficient WMS for \ufotwo{} to sample the remaining substructures.
  \item Let $\mathcal{P}_\exists$ be the union of vocabularies of the existentially quantified formulas
  \begin{equation*}
      \mathcal{P}_\exists := \mathcal{P}_{\extformula_1(x,y)} \cup \mathcal{P}_{\extformula_2(x,y)} \cup \dots \cup \mathcal{P}_{\extformula_m(x,y)}. 
  \end{equation*}
  We further decomposed the sampling probability $\pro[\structure\mid\sentence_T\land \bigwedge_{i\in[n]}\eta_i(e_i)]$ into
  \begin{equation*}
    \pro\left[\structure\mid\sentence_T\land\bigwedge_{i\in[n]}\eta_i(e_i)\land\structure_\exists\right]\cdot \pro\left[\structure_\exists\mid\sentence_T\land \bigwedge_{i\in[n]}\eta_i(e_i)\right],
  \end{equation*}
  where $\structure_\exists$ is a $\mathcal{P}_\exists$-structure over $\domain$.
  It decomposes the conditional sampling problem of $\structure$ into two subproblems---one is to sample $\structure_\exists$ and the other sample the remaining substructures conditional on $\structure_\exists$.
  The advantage of this decomposition is that the latter subproblem can be reduced into a sampling problem on a \ufotwo{} sentence, since all existentially-quantified formulas have been satisfied with $\structure_\exists$.
  The first subproblem for sampling $\structure_\exists$ can be solved by a similar algorithm of $\code{DRSampler}$.
  In this algorithm, the 2-tables used to partition cells are now defined over $\mathcal{P}_\exists$, whose size is exponentially smaller than the one in the original algorithm.
  As a result, the enumeration of 2-table configurations in Algorithm~\ref{alg:drsampler} will be exponentially faster.
  \item We cached $\mathcal{W}_\vecn$, the weight $\typeweight{\tau^i}$ of all 1-types and the weight $\typeweight{\pi^i}$ of all 2-tables, which are widely used in our WMS. 
\end{itemize}

\subsubsection{The Function $\code{ExSat}(\cdot, \cdot)$}
\label{sub:exsat}

The pseudo-code of $\code{ExSat}$ is presented in Algorithm~\ref{alg:exsat}.

\begin{algorithm}[H] 
  \caption{ExSat($\vecg, \eta$)} 
  \label{alg:exsat}
  \begin{algorithmic}[1]
  \State Decompose $\vecg$ into $\{g_{\eta^i,\pi^j}\}_{i\in[N_c],j\in[N_b]}$
  \State $(\beta, \tau)\gets (\eta)$
  \State // Check the coherence of 2-tables
  \For{$i\in[N_u]$}
    \For{$j\in[N_b]$}
        \State // $\tau(\eta^i)$ is the 1-type in $\eta^i$
        \If{$\pi^j$ is not coherent with $\tau(\eta^i)$ and $\tau$ and $g_{\eta^i, \pi^j} > 0$}
            \State \Return \code{False}
        \EndIf
    \EndFor
  \EndFor
  \State // Check the satisfaction of existentially-quantified formulas
  \State $\forall j\in[N_b], h_{\pi^j}\gets\sum_{i\in[N_c]} g_{\eta^i, \pi^j}$
  \For{$Z_k(x)\in\beta$}
    \For{$j\in[N_b]$}
        \If{$R_k(x,y)\in\pi^j$ and $h_{\pi^j} > 0$}
            \State \Return \code{True}
        \EndIf
    \EndFor
  \EndFor
  \State \Return \code{False}
  \end{algorithmic}
\end{algorithm}

\subsubsection{A Lifted WMS for \fotwo{} with Cardinality Constraints}
\label{sub:wmscc}

We present the modified sampling algorithm for \fotwo{} with cardinality constraints.
The changes from the original WMS for \fotwo{} are highlighted by the blue lines.

\begin{algorithm}[!htb] 
  \caption{$\code{WMS}(\sentence, \Upsilon, \domain, \weight, \negweight)$}
  % \SetKwProg{generate}{Function \emph{generate}}{}{end}
  \begin{flushleft}
    \textbf{INPUT:} An \fotwo{} sentence $\sentence$ of the form \eqref{eq:scott_form}, a set of cardinality constraints $\Upsilon$ of the form~\eqref{eq:cc}, a domain $\domain=\{e_i\}_{i\in[n]}$ of size $n$, a weighting $(\weight, \negweight)$\\
    \textbf{OUTPUT:} A model $\mu$ of $\sentence\land\Upsilon$ over $\domain$\\
  \end{flushleft}
  \begin{algorithmic}[1]
    \State Construct $\sentence$ to $\sentence_T$ by~\eqref{eq:tseitin_reduction}
    \State $\forall i\in[n], \beta_i(x)\gets\bigwedge_{k\in[m]}Z_k(x)$
    \State $\generalsentence\gets\sentence_T\land\bigwedge_{i\in[n]}\beta_i(e_i)$
    \State \textcolor{blue}{$\left(\tau_i\right)_{i\in[n]}\gets\code{OneTypeSampler}(\generalsentence, \Upsilon, \domain, \weight, \negweight)$}
    \State $\forall i\in[n], \eta_i\gets (\beta_i, \tau_i)$
    \State \textcolor{blue}{$\mu\gets \code{DRSampler}(\sentence_T, \Upsilon, \domain,\weight, \negweight,(\eta_i)_{i\in[n]})$}
    \State \Return $\mu$
  \end{algorithmic}
\end{algorithm}

\begin{algorithm}[!htb] 
  \caption{OneTypeSampler($\generalsentence, \Upsilon, \domain, \weight, \negweight$)} 
  \label{alg:ccsampling_cell_types}
  \begin{algorithmic}[1]
  \State \textcolor{blue}{$W\gets \symwfomc(\generalsentence\land\Upsilon, \domain, \weight, \negweight)$}
  \State Obtain the blocks $B_{\beta^1}, B_{\beta^2}, \dots, B_{\beta^{2^m}}$ from $\generalsentence$
%   \State // $\code{Prod}(\dots)$ is the Cartesian product
  \For{$\left(\vecn_{\beta^i}\right)_{i\in[2^m]}\in\code{Prod}\left(\mathcal{T}_{|B_{\beta^1}|, N_u}, \dots, \mathcal{T}_{|B_{\beta^{2^m}}|, N_u}\right)$}
    \State $\vecn \gets \bigoplus_{i\in[2^m]} \vecn_{\beta^i}$
    % \State // Compute the weighted sum over all models with the cell configuration
    \State \textcolor{blue}{Compute $\sumweight_{\vecn, \Upsilon}$}
    % \State // $\binom{|B_{\beta^t}|}{\vecn_{\beta^t}}$'s are multinomial coefficients
    \State \textcolor{blue}{$W' \gets \sumweight_{\vecn,\Upsilon} \cdot \prod_{t=1}^{2^m}\binom{|B_{\beta^t}|}{\vecn_{\beta^t}}$}
    \If{$\code{Uniform}(0, 1) < \frac{W'}{W}$}
      \State $\vecn^* \gets \vecn$
      \State \textbf{break}
    \Else
      \State $W \gets W - W'$
    \EndIf
  \EndFor
  \For{$i\in[2^m]$}
    \State Fetch the cell configuration $\vecn^*_{\beta^i}$ in $\beta^i$ from $\vecn^*$
    \State Randomly partition $B_{\beta^i}$ into $\{C_{\beta^i, \tau^j}\}_{j\in[N_u]}$ according to $\vecn^*_{\beta^i}$
    \For{$j\in[N_u]$}
      \State Assign the 1-type $\tau^j$ to all elements in $C_{\beta^i, \tau^j}$ 
    \EndFor
  \EndFor
  \end{algorithmic}
\end{algorithm}

\begin{algorithm}[!htb]
  \caption{$\code{DRSampler}(\sentence_T, \Upsilon, \domain, \weight, \negweight, (\eta_i)_{i\in[n]})$} 
  \label{alg:ccdrsampler}
  % \SetKwProg{generate}{Function \emph{generate}}{}{end}
  % \begin{flushleft}
  %   \textbf{INPUT:} A sentence $\sentence_T$ of the form \eqref{eq:general_sentence}, a domain $\domain=\{e_i\}_{i\in[n]}$, a weighting $(\weight, \negweight)$ and the cell type $\eta_i = (\tau_i,\beta_i)$ of each domain element $e_i$\\
  %   \textbf{OUTPUT:} A model $\mu$ of $\generalsentence$\\
  % \end{flushleft}
  \begin{algorithmic}[1]
    % \State $n\gets |\domain|$
    \If{$n = 1$}
    \State \Return $\emptyset$
    \EndIf
    \If{only the block of type $\beta =\top$ is nonempty}
      \State \textcolor{blue}{\Return a model $\mu$ sampled by a \ufotwo{} WMS from $\forall x\forall y: \fotwoformula(x,y)\land\bigwedge_{i\in[n]}\tau_i(e_i)\land\Upsilon$ over $\domain$ under $(\weight, \negweight)$}
    \EndIf
    \State Choose $t\in[n]$; $\mu \gets \tau_t(e_t)$; $\domain'\gets\domain \setminus \{e_t\}$
    \State Get the cell configuration $\vecn = \left(n_{\eta^i}\right)_{i\in[N_c]}$ of $\left(\eta_i\right)_{i\in[n]\setminus\{t\}}$
    \State \textcolor{blue}{$W \gets \sumweight_{\vecn, \Upsilon}$}
    \State $\forall i\in[N_c], T_i \gets \mathcal{T}_{n_{\eta^i}, N_b}$
    \For{$\left(\vecg_{\eta^i}\right)_{i\in[N_c]}\gets\code{Prod}(\mathcal{T}_{n_{\eta^1}, N_b}, \dots, \mathcal{T}_{n_{\eta^{N_c}}, N_b})$}
    \State $\vecg\gets\bigoplus_{i\in[N_c]}\vecg_{\eta^i}$
      \If{\textcolor{blue}{$\code{ExSat}\left(\vecg, \eta_t, \Upsilon\right)$}}
        \State Get the new cell configuration $\vecn'$ w.r.t. $\vecg$ by \eqref{eq:configuration_reduction}
        \State \textcolor{blue}{Get the new cardinality constraints $\Upsilon'$ w.r.t. $\vecg$ by \eqref{eq:reduced_cc}}
        \State \textcolor{blue}{Compute $\sumweight_{\vecn', \Upsilon}$}
        \State \textcolor{blue}{$W'\gets \sumweight_{\vecn', \Upsilon}\cdot \typeweight{\tau_t} \cdot  \prod_{i\in[N_c]} \binom{n_{\eta^i}}{\vecg_{\eta^i}}\cdot \mathbf{w}^{\vecg_{\eta^i}}$}
        \If{$\code{Uniform}(0,1) < \frac{W'}{W}$}
          \State $\vecg^* \gets \vecg$
          \State \textbf{break}
        \Else
          \State $W \gets W - W'$
        \EndIf
      \EndIf
    \EndFor
    \State Obtain the cell partition $\{C_{\eta^i}\}_{i\in[N_c]}$ from $(\eta_i)_{i\in[n]\setminus\{t\}}$
    \For{$i\in[N_c]$}
      \State Fetch the cardinality vector $\vecg^*_{\eta^i}$ of $\eta^i$ from $\vecg^*$
      \State Randomly partition the cell $C_{\eta^i}$ into $\left\{G_{\eta^i, \pi^j}\right\}_{j\in[N_b]}$ according to $\vecg^*_{\eta^i}$
      \For{$j\in[N_b]$}
        \State $\mu \gets \mu \cup \left\{\pi^j(e_t, e)\right\}_{e\in G_{\eta^i, \pi^j}}$
        \State $\forall e_s\in G_{\eta^i, \pi^j}, \eta_s' \gets \relaxcelltype{\eta_s}{\pi^j}$
      \EndFor
    \EndFor
    \State \textcolor{blue}{Obtain the reduced cardinality constraints $\Upsilon'$ w.r.t. $\vecg^*$ by~\eqref{eq:reduced_cc}}
    \State \textcolor{blue}{$\mu\gets \mu \cup \code{DRSampler}(\sentence_T, \Upsilon', \domain', \weight, \negweight, \left(\eta_i'\right)_{i\in[n-1]})$}
    \State \Return $\mu$
  \end{algorithmic}
\end{algorithm}

\clearpage

\bibliographystyle{IEEEtranN}
\bibliography{SRL}

\end{document}